\documentclass[journal]{IEEEtran}

\usepackage{times}
\usepackage{helvet}
\usepackage{courier}
\usepackage{booktabs}
\usepackage{floatrow}
\usepackage{url}
\usepackage{floatrow}
\usepackage{graphicx} 
\usepackage{amsmath, amssymb, amsthm}
\usepackage{dsfont}
\usepackage{algorithm}
\usepackage{algorithmic}
\usepackage{subfigure}
\usepackage[mathlines]{lineno}
\usepackage{amsmath}
\usepackage{multirow}
\usepackage{xcolor}
\usepackage[colorlinks, linkcolor=black, urlcolor=blue,anchorcolor=black, citecolor=blue]{hyperref}
\usepackage{cite}

\newtheorem{theorem}{Theorem}
\newtheorem{lemma}{Lemma}

\ifCLASSINFOpdf
\else
\fi
\begin{document}
%
\title{Automatic Subspace Learning via Principal Coefficients Embedding}
%
%
%

\author{Xi~Peng,
        Jiwen~Lu,~\IEEEmembership{Senior Member,~IEEE,}
        Zhang~Yi,~\IEEEmembership{Fellow,~IEEE}
        and Rui~Yan,~\IEEEmembership{Member,~IEEE,}                
\thanks{Manuscript received ***; revised ***; accepted ***. This work was supported by A*STAR Industrial Robotics Program - Distributed Sensing and Perception under SERC grant 1225100002 and the National Natural Science Foundation of China under Grant 61432012.}
\thanks{Xi~Peng is with Institute for Infocomm Research, Agency for Science, Technology and Research (A*STAR), Singapore 138632. E-mail: pangsaai@gmail.com.}
\thanks{Jiwen~Lu is with the Department of Automation, Tsinghua University, Beijing 100084, China. E-mail: lujiwen@mail.tsinghua.edu.cn.}
\thanks{Zhang~Yi and Rui~Yan are with College of Computer Science, Sichuan University, Chengdu, 610065, China. E-mail: zhangyi@scu.edu.cn; yanrui2006@gmail.com.}
\thanks{Corresponding author: Rui Yan.}}

%
%

\markboth{IEEE Transactions on Cybernetics}{}
%



\maketitle

\begin{abstract}
In this paper, we address two challenging problems in unsupervised subspace learning: 1) how to automatically identify the feature dimension of the learned subspace (i.e., automatic subspace learning), and 2) how to learn the underlying subspace in the presence of Gaussian noise (i.e., robust subspace learning). We show that these two problems  can be simultaneously solved by proposing a new method (called principal coefficients embedding, PCE). For a given data set $\mathbf{D}\in \mathds{R}^{m\times n}$, PCE recovers a clean data set $\mathbf{D}_{0}\in \mathds{R}^{m\times n}$ from $\mathbf{D}$ and simultaneously learns a global reconstruction relation $\mathbf{C}\in \mathbf{R}^{n\times n}$ of $\mathbf{D}_{0}$. By preserving $\mathbf{C}$ into an $m^{\prime}$-dimensional space, the proposed method obtains a projection matrix that can capture the latent manifold structure of $\mathbf{D}_{0}$, where $m^{\prime}\ll m$ is automatically determined by the rank of $\mathbf{C}$ with theoretical guarantees. PCE has three advantages: 1) it can automatically determine the feature dimension even though data are sampled from a union of multiple linear subspaces in presence of the Gaussian noise; 2) Although the objective function of PCE only considers the Gaussian noise, experimental results show that it is robust to the non-Gaussian noise  (\textit{e.g.}, random pixel corruption) and real disguises; 3) Our method has a closed-form solution and can be calculated very fast. Extensive experimental results show the superiority of PCE on a range of databases with respect to the classification accuracy, robustness and efficiency. 
\end{abstract}


%
\IEEEpeerreviewmaketitle

\section{Introduction}
\label{sec1}

Subspace learning or metric learning aims to find a projection matrix $\mathbf{\Theta}\in \mathds{R}^{m \times m^{\prime}}$ from the training data $\mathbf{D}^{m\times n}$, so that the high-dimensional datum $\mathbf{y}\in\mathds{R}^{m}$ can be  transformed into a low-dimensional space via $\mathbf{z}=\mathbf{\Theta}^{T}\mathbf{y}$. Existing subspace learning methods can be roughly divided into three categories: supervised, semi-supervised, and unsupervised. Supervised method incorporates the class label information of $\mathbf{D}$ to obtain discriminative features. The well-known works include but not limit to linear discriminant analysis~\cite{Fisher1936}, neighbourhood components analysis~\cite{Jacob2004}, and their variants such as~\cite{Chen2005,Lu2014,Wang2014:Fisher,Yuan2016:Cong}. Moreover, Xu et\ al.~\cite{Xu2014:PAMI} recently propose to formulate the problem of supervised multiple view subspace learning as one multiple source communication system, which provide a novel insight to the community. Semi-supervised methods~\cite{Cai2007,Yan2009semi,Xu2014:Large} utilize limited labeled training data as well as unlabeled ones for better performance. Unsupervised methods seek a low-dimensional subspace without using any label information of training samples. Typical methods in this category include Eigenfaces~\cite{Turk1991}, neighbourhood preserving embedding (NPE)~\cite{He2005}, locality preserving projections (LPP)~\cite{He2005Lap}, sparsity preserving projections (SPP)~\cite{Qiao2010} or known as L1-graph~\cite{Cheng2010}, and multi-view intact space learning~\cite{Xu2015:multi}. For these subspace learning methods, Yan~\emph{et\ al.}~\cite{Yan2007} have shown that most of them can be unified into the framework of graph embedding, \emph{i.e.}, low dimensional features can be achieved by embedding some desirable properties (described by a similarity graph) from a high-dimensional space into a low-dimensional one. By following this framework, this paper focuses on unsupervised subspace learning, \emph{i.e.}, dimension reduction considering the unavailable label information in training data.

Although a large number of subspace learning methods have been proposed, less works have investigated the following two challenging problems simultaneously: 1) how to automatically determine the dimension of the feature space, referred to as \textit{automatic subspace learning}, and 2) how to immune the influence of corruptions, referred to as \textit{robust subspace learning}. 

Automatic subspace learning involves the technique of dimension estimation which aims at identifying the number of features necessary for the learned low-dimensional subspace to describe a data set. In previous studies, most existing methods manually set the feature dimension by exploring all possible values based on the classification accuracy. Clearly, such a strategy is time-consuming and easily overfits to the specific data set. In the literature of manifold learning, some dimension estimation methods have been proposed, \emph{e.g.}, spectrum analysis based methods~\cite{Polito2001,Yan2007A}, box-counting based methods~\cite{Kegl2002}, fractal-based methods~\cite{Camastra2002,Mo2012}, tensor voting~\cite{Mordohai2010}, and neighbourhood smoothing \cite{Carter2010}. Although these methods have achieved some impressive results, this problem is still far from solved due to the following limitations: 1) these existing methods may work only when data are sampled in a uniform way and data are free to corruptions, as pointed out by Saul and Roweis~\cite{Saul2003}; 2) most of them can accurately estimate the intrinsic dimension of a single subspace but would fail to work well for the scenarios of multiple subspaces, especially, in the case of the dependent or disjoint subspaces; 3) although some dimension estimation techniques can be incorporated with  subspace learning algorithms, it is more desirable to design a method that can not only  automatically identify the feature dimension but also reduce the dimension of data.

Robust subspace learning aims at identifying underlying subspaces even though the training data $\mathbf{D}$ contains gross corruptions such as the Gaussian noise. Since $\mathbf{D}$ is corrupted by itself, accurate prior knowledge about the desired geometric properties is hard to be learned from $\mathbf{D}$. Furthermore, gross corruptions will make dimension estimation more difficult. This so-called robustness learning problem has been challenging in machine learning and computer vision. One of the most popular solutions is recovering a clean data set from inputs and then performing dimension reduction over the clean data. Typical methods include the well-known principal component analysis (PCA) which achieves robust results by removing the bottom eigenvectors corresponding to the smallest eigenvalues. However, PCA can achieve a good result only when data are sampled from a single subspace and only contaminated by a small amount of noises. Moreover, PCA needs specifying a parameter (\emph{e.g.}, 98\% energy) to distinct the principal components from the minor ones. To improve the robustness of PCA, Candes~\emph{et\ al.} recently proposed robust PCA (RPCA)~\cite{Candes2011} which can handle the sparse corruption and has achieved a lot of success~\cite{He2011:Robust,Peng2012RASL,Zhao2014robust,Nie2014,Liu2015:Large,Hsu2011:RMD}. However, RPCA directly removes the errors from the input space, which cannot obtain the low-dimensional features of inputs. Moreover, the computational complexity of RPCA is too high to handle the large-scale data set with very high dimensionality. Bao~\emph{et\ al.}~\cite{Bao2013} proposed an algorithm which can handle the gross corruption. However, they did not explore the possibility to automatically determine feature dimension. Tzimiropoulos~\emph{et\ al.}~\cite{Tzimiropoulos2012} proposed a subspace learning method from image gradient orientations by replacing pixel intensities of images with gradient orientations. Shu et\ al.~\cite{Shu2014:Robust} proposed to impose the low-rank constraint and group sparsity on the reconstruction coefficients under orthonormal subspace so that the Laplacian noise can be identified. Their method outperforms a lot of popular methods such as Gabor features in illumination- and occlusion-robust face recognition. Liu and Tao have recently carried out a series of comprehensive works to discuss how to handle various noises, \textit{e.g.}, Cauchy noise~\cite{Liu2014:Onrobustness}, Laplacian noise~\cite{Liu2015:On}, and noisy labels~\cite{Liu2016:Class}. Their works provide some novel theoretical explanations towards understanding the role of these errors. Moreover, some recent developments have been achieved in the field of subspace clustering~\cite{Favaro2011,Vidal2014,Peng2015robust,Jones2014,Xiao2014,Yu2015:GTD,Peng2016:Con}, which use $\ell_1$-, $\ell_2$-, or nuclear-norm based representation to achieve robustness.

In this paper, we proposed a robust unsupervised subspace learning method which can automatically identify the number of features. The proposed method, referred to as principal coefficients embedding (PCE), formulates the possible corruptions as a term of an objective function so that a clean data set $\mathbf{D}_{0}$ and the corresponding reconstruction coefficients $\mathbf{C}$ can be simultaneously learned from the training data $\mathbf{D}$. By embedding $\mathbf{C}$ into an $m^{\prime}$-dimensional space, PCE obtains a projection matrix $\mathbf{\Theta}^{m\times m^{\prime}}$, where $m^{\prime}$ is adaptively determined by the rank of $\mathbf{C}$ with theoretical guarantees. 

PCE is motivated by a recent work in subspace clustering~\cite{Favaro2011,Peng2015Connections} and the well-known locally linear embedding (LLE) method~\cite{Roweis2000}. The former motivates us the way to achieve robustness, \textit{i.e.}, the errors such as the Gaussian noise can be mathematically formulated as a term into the objective function and thus the errors can be explicitly removed. The major differences between \cite{Favaro2011} and PCE are: 1) \cite{Favaro2011} is proposed for clustering, whereas PCE is for dimension reduction; 2) the objective functions are different. PCE is based on the Frobenius norm instead of the nuclear norm, thus resulting a closed-form solution and avoiding iterative optimization procedure; 3) PCE can automatically determine the feature dimension, whereas \cite{Favaro2011} does not investigate this challenging problem. LLE motivated us the way to estimate feature dimension even though it does not overcome this problem. LLE is one of most popular  dimension reduction methods, which encodes each data point as a linear combination of its neighborhood and preserves such reconstruction relationship into different projection spaces. LLE implies the possibility to estimate the feature dimension using the size of neighborhood of data points. However, this parameter needs to be specified by users rather than automatically learning from data. Thus, LLE still suffers from the issue of dimension estimation. Moreover, the performance of LLE would be degraded when the data is contaminated by noises. The contributions of this work are summarized as follows: 

\begin{itemize}
  \item The proposed method (\emph{i.e.}, PCE) can handle the Gaussian noise that probably exists into data with theoretical guarantees. Different from the existing dimension reduction methods such as L1-Graph, PCE formulates the corruption into its objective function and only calculates the reconstruction coefficients using a clean data set instead of the original one. Such a formulation can perform data recovery and improve the robustness of PCE;
  \item Unlike previous subspace learning methods, PCE can automatically determine the feature dimension of the learned low-dimensional subspace. This largely reduces the efforts for finding an optimal dimension and thus PCE is more competitive in practice;
  \item PCE is computationally efficient, which only involves performing singular value decomposition (SVD) over the training data set one time. 
\end{itemize}

The rest of this paper is organized as follows. Section~\ref{sec2} briefly introduces some related works. Section~\ref{sec3} presents our proposed algorithm. Section~\ref{sec4} reports the experimental results and Section~\ref{sec5} concludes this work.

\begin{table}[!t]
\caption{Some used notations.}
\label{tab1}
\begin{center}
\begin{small}
\begin{tabular}{ll}
\toprule
Notation & Definition\\
\midrule
$n$ & the number of data points\\
$n_{i}$ & data size of the $i$-th subject\\
$m$ & the dimension of input\\
$m^{\prime}$ & the dimension of feature space \\
$s$ & the number of subject\\
$r$ & the rank of a given matrix\\
$\mathbf{y}\in\mathds{R}^{m}$ &  a given testing sample\\
$\mathbf{z}\in\mathds{R}^{m^{\prime}}$ & the low-dimensional feature\\
$\mathbf{D}=[\mathbf{d}_1, \mathbf{d}_2, \ldots, \mathbf{d}_{n}]$ &  training data set\\
$\mathbf{D}=\mathbf{U} \mathbf{\Sigma} \mathbf{V}^{T}=\mathbf{U}_{r} \mathbf{\Sigma}_{r} \mathbf{V}^{T}_{r}$ & full and skinny SVD of $\mathbf{D}$\\
$\mathbf{D}_{0}\in\mathds{R}^{m\times n}$ & the desired clean data set\\
$\mathbf{E}\in\mathds{R}^{m\times n}$ & the errors existing into $\mathbf{D}$\\
$\mathbf{C}=[\mathbf{c}_1, \mathbf{c}_2, \ldots, \mathbf{c}_{n}]$ & the representation of $\mathbf{D}_{0}$\\
$\sigma_{i}(\mathbf{C})$ & the $i$-th singular value of $\mathbf{C}$\\
$\mathbf{\Theta}\in \mathds{R}^{m\times m^{\prime}}$ & the  projection matrix\\
\bottomrule
\end{tabular}
\end{small}
\end{center}
\end{table}

\section{Related Works}
\label{sec2}

\subsection{Notations and Definitions}
\label{sec2.1}

In the following, \textbf{lower-case bold letters} represent column vectors and \textbf{UPPER-CASE BOLD ONES} denote matrices. $\mathbf{A}^T$ and $\mathbf{A}^{\dag}$ denote the transpose and pseudo-inverse of the matrix $\mathbf{A}$, respectively. $\mathbf{I}$ denotes the identity matrix. 

For a given data matrix $\mathbf{D} \in \mathds{R}^{m\times n}$, the Frobenius norm of $\mathbf{D}$ is defined as
\begin{equation}
	\label{eq2.1}
	\|\mathbf{D}\|_{F} = \sqrt{trace(\mathbf{D}\mathbf{D}^{T})}=\sqrt{\sum_{i=1}^{\min\{m,n\}}\sigma_{i}^{2}(\mathbf{D})},
\end{equation}
where $\sigma_{i}(\mathbf{D})$ denotes the $i$-th singular value of $\mathbf{D}$.

The full Singular Value Decomposition (SVD) and the skinny SVD of $\mathbf{D}$ are defined as $\mathbf{D}=\mathbf{U} \mathbf{\Sigma} \mathbf{V}^{T}$ and
$\mathbf{D}=\mathbf{U}_{r} \mathbf{\Sigma}_{r} \mathbf{V}_{r}^{T}$, where $\mathbf{\Sigma}$ and $\mathbf{\Sigma}_{r}$ are in descending order. $\mathbf{U}_{r}$, $\mathbf{V}_{r}$ and $\mathbf{\Sigma}_{r}$ consist of the top (\emph{i.e.}, largest) $r$ singular vectors and singular values of $\mathbf{D}$.~Table~\ref{tab1} summarizes some notations used throughout the paper.

\subsection{Locally Linear Embedding}
\label{sec2.2}

In~\cite{Yan2007}, Yan~\emph{et\ al.} have shown that most unsupervised, semi-supervised, and supervised subspace learning methods can be unified into a framework known as graph embedding. Under this framework, subspace learning methods obtain low-dimensional features by preserving some desirable geometric relationships from a high-dimensional space into a low-dimensional one. Thus, the performance of subspace learning largely depends on the identified relationship which is usually described by a similarity graph (\emph{i.e.}, affinity matrix). In the graph, each vertex corresponds to a data point and the edge weight denotes the similarity between two connected points. There are two popular ways to measure the similarity among data points,~\emph{i.e.}, pairwise distance such as Euclidean distance~\cite{He2003} and linear reconstruction coefficients introduced by LLE~\cite{Roweis2000}.

For a given data matrix $\mathbf{D}=[\mathbf{d}_{1}, \mathbf{d}_{2}, \ldots,\mathbf{d}_{n}]$, LLE solves the following problem:
\begin{equation}
	\label{eq2.2}
	\min_{\mathbf{c}_{i}}\sum_{i=1}^{n}\|\mathbf{d}_{i}-\mathbf{B}_{i}\mathbf{c}_{i}\|_{2},\hspace{3mm}\mathrm{s.t.}\hspace{1mm}\sum_{j}c_{ij}=1,
\end{equation}
where $\mathbf{c}_{i}\in \mathds{R}^{p}$ is the linear representation of $\mathbf{d}_{i}$ over $\mathbf{B}_{i}$, $c_{ij}$ denotes the $j$-th entry of $\mathbf{c}_{i}$, and $\mathbf{B}_{i}\in \mathds{R}^{m\times p}$ consists of $p$ nearest neighbors of $\mathbf{d}_{i}$ that are chosen from the collection of $[\mathbf{d}_{1}, \ldots, \mathbf{d}_{i-1}, \mathbf{d}_{i+1},\ldots,\mathbf{d}_{n}]$ in terms of Euclidean distance.

By assuming the reconstruction relationship $\mathbf{c}_{i}$ is invariant to ambient space, LLE obtains the low-dimensional features $\mathbf{Y}\in \mathds{R}^{ m^{\prime}\times n}$ of $\mathbf{D}$ by
\begin{equation}
	\label{eq2.3}
	\min_{\mathbf{Y}}\|\mathbf{Y} - \mathbf{Y}\mathbf{W}\|_{F}^{2}, \hspace{3mm} \mathrm{s.t.}\hspace{1mm}\mathbf{Y}^{T}\mathbf{Y}=\mathbf{I},
\end{equation}
where $\mathbf{W}=[\mathbf{w}_{1}, \mathbf{w}_{2},\ldots,\mathbf{w}_{n} ]$ and the nonzero entries of $\mathbf{w}_{i}\in \mathds{R}^{n}$ corresponds to $\mathbf{c}_{i}$.

However, LLE cannot handle the out-of-sample data that are not included into $\mathbf{D}$. To solve this problem,  NPE~\cite{He2003} calculates the projection matrix $\mathbf{\Theta}$ instead of $\mathbf{Y}$ by replacing $\mathbf{Y}$ with $\mathbf{\Theta}^{T}\mathbf{D}$ into (\ref{eq2.3}).

\subsection{L1-Graph}
\label{sec2.3}

By following the framework of LLE and NPE, Qiao~\emph{et\ al.}~\cite{Qiao2010} and Cheng~\emph{et\ al.}~\cite{Cheng2010} proposed SPP and L1-graph, respectively. The methods sparsely encode each data points by solving the following sparse coding problem:
\begin{equation}
	\label{eq2.4}
	\min_{\mathbf{c}_{i}}\|\mathbf{d}_{i}-\mathbf{D}_{i}\mathbf{c}_{i}\|_{2}+\lambda\|\mathbf{c}_{i}\|_{1},
\end{equation}
where $\mathbf{D}_{i}=[\mathbf{d}_{1},\ldots,\mathbf{d}_{i-1},\mathbf{0},\mathbf{d}_{i+1},\ldots,\mathbf{d}_{n}]$ and (\ref{eq2.4}) can be solved by many $\ell_1$-solvers~\cite{Yang2010,Liu2015:L1}.

After obtaining $\mathbf{C}\in \mathds{R}^{n\times n}$, SPP and L1-graph embed $\mathbf{C}$ into the feature space by following NPE. The advantage of sparsity based subspace methods is that they can automatically determine the neighbourhood for each data point without the parameter of neighbourhood size. Inspired by the success of SPP and L1-graph, a number of spectral embedding  methods~\cite{Zhang2013,Tang2015CYB,Lu2015,Tao2013,Hou2014,Chen2014CYB} have been proposed. However, these methods including L1-graph and SPP have still required specifying the dimension of feature space.

\subsection{Robust Principal Component Analysis}
\label{sec2.4}

RPCA~\cite{Candes2011} is proposed to improve the robustness of PCA, which solves the following optimization problem:
\begin{equation}
	\label{eq2.5}
	\min_{\mathbf{D}_{0},\mathbf{E}} \mathrm{rank}(\mathbf{D}_{0})+\lambda\|\mathbf{E}\|_{0}\hspace{3mm}\mathrm{s.t.}\hspace{1mm}\mathbf{D}=\mathbf{D}_{0}+\mathbf{E},
\end{equation}
where $\lambda>0$ is the parameter to balance the possible corruptions and the desired clean data, and $\|\cdot\|_{0}$ is $\ell_0$-norm to count the number of nonzero entries of a given matrix or vector. 

Since the rank operator and $\ell_0$-norm are non-convex and discontinuous, ones usually relax them with nuclear norm and $\ell_1$-norm~\cite{Recht2010}. Then, (\ref{eq2.5}) is approximated by 
\begin{equation}
	\label{eq2.6}
	\min_{\mathbf{D}_{0},\mathbf{E}} \|\mathbf{D}_{0}\|_{\ast}+\lambda\|\mathbf{E}\|_{1}\hspace{3mm}\mathrm{s.t.}\hspace{1mm}\mathbf{D}=\mathbf{D}_{0}+\mathbf{E},
\end{equation}
where $\|\mathbf{D}\|_{\ast}=trace(\sqrt{\mathbf{D}^{T}\mathbf{D}})=\sum_{i=1}^{\min\{m,n\}}\sigma_{i}(\mathbf{D})$ denotes the nuclear norm of $\mathbf{D}$ and $\sigma_{i}(\mathbf{D})$ is the $i$-th singular value of $\mathbf{D}$.

\section{Principal Coefficients Embedding for Unsupervised Subspace Learning}
\label{sec3}

In this section, we propose an unsupervised algorithm for  subspace learning, \emph{i.e.}, principal coefficients embedding (PCE). The method not only can achieve robust results but also can automatically determine the feature dimension.

For a given data set $\mathbf{D}$ containing the errors $\mathbf{E}$, PCE achieves robustness and dimension estimation in two steps: 1) the first step achieves the robustness by recovering a clean data set $\mathbf{D}_{0}$ from $\mathbf{D}$ and building a similarity graph $\mathbf{C}\in \mathds{R}^{n\times n}$ with the reconstruction coefficients of $\mathbf{D}_{0}$, where $\mathbf{D}_{0}$ and $\mathbf{C}$ are jointly learned by solving a SVD problem; 2) the second step automatically estimates the feature dimension using the rank of $\mathbf{C}$ and learns the projection matrix $\mathbf{\Theta}\in \mathds{R}^{m\times m^{\prime}}$ by embedding $\mathbf{C}$ into an $m^{\prime}$-dimensional space. In the following, we will introduce these two steps in details.

\subsection{Robustness Learning}

For a given training data matrix $\mathbf{D}$, PCE removes  the corruption $\mathbf{E}$ from $\mathbf{D}$ and then linearly encodes the recovered clean data set $\mathbf{D}_{0}$ over itself. The proposed objective function is as follows:
\begin{equation}
	\label{eq3.1}
	\min_{\mathbf{C},\mathbf{D}_{0},\mathbf{E}}\frac{1}{2}\|\mathbf{C}\|_{F}^{2}+\frac{\lambda}{2}\|\mathbf{E}\|_{p}^{2}
\hspace{3mm} \mathrm{s.t.} \hspace{1mm} \underbrace{\mathbf{D}=\mathbf{D}_{0}+\mathbf{E}}_{\text{Robustness}}, \underbrace{\mathbf{D}_{0}=\mathbf{D}_{0}\mathbf{C}}_{\text{self-expression}}
\end{equation}

The proposed objective function mainly considers the constraints on the representation $\mathbf{C}$ and the errors $\mathbf{E}$. We enforce Frobenius norm on $\mathbf{C}$ because some recent works have shown that the Frobenius norm based representation is more computationally efficient than the $\ell_1$- and nuclear-norm based representation while achieving competitive performance in face recognition~\cite{Zhang2011} and subspace clustering~\cite{Peng2015robust}. Moreover, Frobenius-norm based representation has shared some desirable properties with nuclear-norm based representation as shown in our previous theoretical studies~\cite{Zhang2014fLRR,Peng2015Connections}.

The term $\mathbf{D}_{0}=\mathbf{D}_{0}\mathbf{C}$ is motivated by the recent development in subspace clustering~\cite{Elhamifar2013,Liu2013}, which can be further derived from the formulation of LLE (\textit{i.e.}, eqn.(\ref{eq2.3})). More specifically, ones reconstruct $\mathbf{D}_{0}$ by itself to obtain this so-called self-expression as the similarity of data set. The major differences between this work and the existing methods are: 1) the objective functions are different. Our method is based on Frobenius norm instead of $\ell_1$- or nuclear-norm; 2) the methods directly project the original data $\mathbf{D}$ into the space spanned by itself, whereas we simultaneously learn a clean data set $\mathbf{D}_{0}$ from $\mathbf{D}$ and compute the self-expression of $\mathbf{D}_{0}$; 3) PCE is proposed for subspace learning, whereas the methods are proposed for clustering. 

By formulating the error $\mathbf{E}$ as a term into our objective function, we can achieve robustness by $\mathbf{D}_{0}=\mathbf{D}-\mathbf{E}$. The constraint on $\mathbf{E}$ (\textit{i.e.}, $\|\cdot\|_{p}$) could be chosen as $\ell_1$-, $\ell_2$-, or $\ell_{2,1}$-norm. Different choices of $\|\cdot\|_{p}$ correspond to different types of noises. For example, $\ell_1$-norm is usually used to formulate the Laplacian noise, $\ell_2$-norm is adopted to describe the Gaussian noise, and $\ell_{2,1}$-norm is used to represent the sample-specified corruption such as outlier~\cite{Liu2013}. Here, we mainly consider the Gaussian noise which is commonly assumed in signal transmission problem. Thus, we have the following objective function:
\begin{equation}
	\label{eq3.1b}
	\min_{\mathbf{C},\mathbf{D}_{0},\mathbf{E}}\frac{1}{2}\|\mathbf{C}\|_{F}^{2}+\frac{\lambda}{2}\|\mathbf{E}\|_{F}^{2}
\hspace{3mm} \mathrm{s.t.} \hspace{1mm} \mathbf{D}=\mathbf{D}_{0}+\mathbf{E}, \mathbf{D}_{0}=\mathbf{D}_{0}\mathbf{C},
\end{equation}
where $\|\mathbf{E}\|_{F}$ denotes the error that follows the Gaussian distribution. It is worthy to point out that, although the above formulation only considers the Gaussian noise, our experimental result show that PCE is also robust to other corruptions such as random pixel corruption (non-additive noise) and real disguises. 

To efficiently solve (\ref{eq3.1b}), we first consider the case of corruption-free, \emph{i.e.}, $\mathbf{E}=\mathbf{0}$. In such a setting, (\ref{eq3.1b}) is simplified as follows:
\begin{equation}
	\label{eq3.2}
	\min_{\mathbf{C}}\|\mathbf{C}\|_{F}\hspace{3mm}\mathrm{s.t.}\hspace{1mm}\mathbf{D}=\mathbf{D}\mathbf{C}.
\end{equation}

Note that, $\mathbf{D}^{\dag}\mathbf{D}$ is a feasible solution to $\mathbf{D}=\mathbf{D}\mathbf{C}$, where $\mathbf{D}^{\dag}$ denotes the pseudo-inverse of $\mathbf{D}$. In~\cite{Peng2015Connections}, an unique minimizer to (\ref{eq3.2}) is given as follows.

\begin{lemma}
\label{thm1}
Let $\mathbf{D}=\mathbf{U}_{r} \mathbf{\Delta}_{r} \mathbf{V}_{r}^{T}$ be the skinny SVD of the data matrix $\mathbf{D}\ne \mathbf{0}$. The unique solution to
\begin{equation}
\label{thm1:equ1}
    \min\hspace{1mm}\|\mathbf{C}\|_{F} \hspace{3mm}
    \mathrm{s.t.} \hspace{1mm} \mathbf{D}= \mathbf{D} \mathbf{C},
\end{equation}
is given by $\mathbf{C}^{\ast} = \mathbf{V}_{r}\mathbf{V}_{r}^{T}$, where $r$ is the rank of $\mathbf{D}$ and $\mathbf{D}$ is a clean data set without any corruptions.
\end{lemma}

\begin{proof}
Let $\mathbf{D}=\mathbf{U} \mathbf{\Delta} \mathbf{V}^{T}$ be the full SVD of $\mathbf{D}$. The pseudo-inverse of $\mathbf{D}$ is $\mathbf{D}^{\dag} =
\mathbf{V}_{r} \mathbf{\Delta}_{r}^{-1} \mathbf{U}_{r}^{T}$. Defining $\mathbf{V}_{c}$ by $\mathbf{V}^T=\left[
\begin{array}{c}
\mathbf{V}_r^T
\\
\mathbf{V}_c^T
\end{array}\right]$ and $\mathbf{V}_{c}^{T} \mathbf{V}_{r} = \mathbf{0}$. To prove that  $\mathbf{C}^{\ast} = \mathbf{V}_{r}\mathbf{V}_{r}^{T}$ is the unique solution to (\ref{thm1:equ1}), two steps are required. 

First, we prove that $\mathbf{C}^{\ast}$ is the minimizer to (\ref{thm1:equ1}), \emph{i.e.}, for any $\mathbf{X}$ satisfying $\mathbf{D}=\mathbf{D}\mathbf{X}$, it must hold that $\|\mathbf{X}\|_{F} \geq \|\mathbf{C}^{\ast}\|_{F}$. Since for any column orthogonal matrix $\mathbf{P}$, it must hold that $\| \mathbf{PM} \|_{F} = \| \mathbf{M} \|_{F}$. Then, we have
\begin{eqnarray}
\label{thm1:equ2}
\|\mathbf{X}\|_{F}
& = & \left\|\left[
\begin{array}{ll}
\mathbf{V}^T_r
\\
\mathbf{V}^T_c
\end{array}
\right] \left[\mathbf{C}^*+\left( \mathbf{X}-\mathbf{C}^{\ast} \right) \right] \right\|_{F}
\notag\\
& = & \left\| \left[
\begin{array}{ll}
\mathbf{V}^T_r \mathbf{C}^{\ast}+\mathbf{V}^T_r \left(\mathbf{X}-\mathbf{C}^{\ast}\right)
\\
\mathbf{V}^T_c \mathbf{C}^{\ast} + \mathbf{V}^T_c(\mathbf{X} - \mathbf{C}^{\ast})
\end{array}
\right] \right\|_{F}.
\end{eqnarray}

As $\mathbf{C}^{\ast}$ satisfies $\mathbf{D} = \mathbf{D} \mathbf{C}^{\ast}$, then $\mathbf{D} \left(\mathbf{X}-\mathbf{C}^{\ast}\right) = \mathbf{0}$, \emph{i.e.}, $\mathbf{U}_{r} \mathbf{\Delta}_{r}
\mathbf{V}_{r}^{T} \left(\mathbf{X}-\mathbf{C}^{\ast}\right) = \mathbf{0}$. Since $\mathbf{U}_{r} \mathbf{\Delta}_{r}\ne \mathbf{0}$, 
$\mathbf{V}^{T}_{r}(\mathbf{X}-\mathbf{C}^{{\ast}})=\mathbf{0}$. Denote $\mathbf{\Gamma} = \mathbf{\Delta}_{r}^{-1}
\mathbf{U}_{r}^{T} \mathbf{D}$, then $\mathbf{C}^{\ast}=\mathbf{V}_{r} \mathbf{\Gamma}$. Because $\mathbf{V}_{c}^{T}\mathbf{V}_{r} = \mathbf{0}$, we have $\mathbf{V}_{c}^{T} \mathbf{C}^{{\ast}} = \mathbf{V}_{c}^{T} \mathbf{V}_{r} \mathbf{\Gamma}
= \mathbf{0}$. Then, it follows that
\begin{equation}
\label{thm1:equ3}
\left\| \mathbf{X} \right\|_{F} = \left\| \left[
\begin{array}{c}
\mathbf{\Gamma}
\\
\mathbf{V}_{c}^{T}(\mathbf{X}-\mathbf{C}^{\ast})
\end{array}
\right] \right\|_{F}.
\end{equation}

Since for any matrixes $\mathbf{M}$ and $\mathbf{N}$ with the same number of columns, it holds that
\begin{equation}
\label{thm1:equ3b}
\left\| \left[
\begin{array}{c}
\mathbf{M}
\\
\mathbf{N}
\end{array}
\right] \right\|_{F}^{2} = \|\mathbf{M}\|_{F}^{2} + \|\mathbf{N}\|_{F}^{2}.
\end{equation}

From (\ref{thm1:equ3}) and (\ref{thm1:equ3b}), we have
\begin{equation}
\label{thm1:equ4}
\left\| \mathbf{X} \right\|_{F}^{2} = \|\mathbf{\Gamma}\|_{F}^{2} +
\|\mathbf{V}_{c}^{T}(\mathbf{X}-\mathbf{C}^{\ast})\|_{F}^{2},
\end{equation}
which shows that $\|\mathbf{X}\|_{F} \geq \|\mathbf{\Gamma}\|_{F}$.

Furthermore, since
\begin{equation}
\label{thm1:equ5}
\|\mathbf{\Gamma}\|_{F}=\|\mathbf{V}_{r} \mathbf{\Gamma}\|_{F} = \| \mathbf{C}^{\ast} \|_{F},
\end{equation}
we have $\|\mathbf{X}\|_{F} \geq \|\mathbf{C}^{\ast}\|_{F}$.

Second, we prove that $\mathbf{C}^{\ast}$ is the unique solution of (\ref{thm1:equ1}). Let $\mathbf{X}$ be another
minimizer, then, $\mathbf{D} = \mathbf{D} \mathbf{X}$ and
$\|\mathbf{X}\|_{F}=\|\mathbf{C}^{\ast}\|_{F}$. From (\ref{thm1:equ4}) and (\ref{thm1:equ5}),
\begin{equation}
\|\mathbf{X}\|_{F}^{2} = \|\mathbf{C}^{\ast}\|_{F}^{2} + \|\mathbf{V}_{c}^{T}
\left(\mathbf{X} - \mathbf{C}^{\ast}\right)\|_{F}^{2}.
\end{equation}

Since $\|\mathbf{X}\|_{F} = \|\mathbf{C}^{\ast}\|_{F}$, it must hold that
$\|\mathbf{V}_{c}^{T}\left(\mathbf{X}-\mathbf{C}^{\ast}\right)\|_{F} = 0$, and then
$\mathbf{V}_{c}^{T}\left(\mathbf{X} - \mathbf{C}^{\ast}\right) = \mathbf{0}$. Together with $\mathbf{V}_{r}^{T}
\left(\mathbf{X} - \mathbf{C}^{\ast}\right) = \mathbf{0}$, this gives $\mathbf{V}^{T}
\left(\mathbf{X} - \mathbf{C}^{\ast}\right) = \mathbf{0}$. Because $\mathbf{V}$ is an orthogonal matrix, it must hold that $\mathbf{X} = \mathbf{C}^{\ast}$. 
\end{proof}

Based on Lemma~\ref{thm1}, the following theorem can be used to solve the robust version of PCE (\textit{i.e.}, $\mathbf{E}\ne \mathbf{0}$).
\begin{theorem}
\label{thm2}
Let $\mathbf{D}=\mathbf{U\Sigma}\mathbf{V}^{T}$ be the full SVD of $\mathbf{D}\in \mathds{R}^{m\times n}$, where the diagonal entries of $\mathbf{\Sigma}$ are in descending order, $\mathbf{U}$ and $\mathbf{V}$ are corresponding left and right singular vectors, respectively. Suppose there exists a clean data set and errors, denoted by $\mathbf{D}_{0}$ and $\mathbf{E}$, respectively. The optimal $\mathbf{C}$ to (\ref{eq3.1b})
is given by $\mathbf{C}^{\ast}=\mathbf{V}_{k}\mathbf{V}_{k}^{T}$, where $\lambda$ is a balanced factor, $\mathbf{V}_{k}$ consists of the first $k$ right singular vectors of $\mathbf{D}$, $k=\mathrm{argmin}_{r} r+\lambda\sum_{i>r}\sigma_{i}^{2}$, and $\sigma_{i}$ denotes the $i$-th diagonal entry of $\mathbf{\Sigma}$.
\end{theorem}

\begin{proof}
(\ref{eq3.1b}) can be rewritten as
\begin{equation}
\label{thm2:equ2}
\min_{\mathbf{D}_{0}, \mathbf{C}}\hspace{1mm}\frac{1}{2}\|\mathbf{C}\|_{F}^{2}+\frac{\lambda}{2}\|\mathbf{D}-\mathbf{D}_{0}\|_{F}^{2} \hspace{3mm}\mathrm{s.t.}\hspace{1mm} \mathbf{D}_{0}=\mathbf{D}_{0}\mathbf{C},
\end{equation}

Let $\mathbf{D}_{0}^{\ast}=\mathbf{U}_{r}\mathbf{\Sigma}_{r}\mathbf{V}_{r}^{T}$ be the skinny SVD of $\mathbf{D}_{0}$, where $r$ is the rank of $\mathbf{D}_{0}$. Let $\mathbf{U}_{c}$ and $\mathbf{V}_{c}$ be the basis that orthogonal to $\mathbf{U}_{r}$ and $\mathbf{V}_{r}$, respectively. Clearly, $\mathbf{I}=\mathbf{V}_{r}\mathbf{V}_{r}^{T}+\mathbf{V}_{c}\mathbf{V}_{c}^{T}$. By Lemma~\ref{thm1}, the representation over the clean data $\mathbf{D}_{0}$ is given by $\mathbf{C}^{\ast}=\mathbf{V}_{r}\mathbf{V}_{r}^{T}$.  Next, we will bridge $\mathbf{V}_{r}$ and $\mathbf{V}$.

Using Lagrange method, we have
\begin{equation}
\label{thm2:equ3}
\mathcal{L}(\mathbf{D}_{0}, \mathbf{C})=\frac{1}{2}\left
\|\mathbf{C}\right\|_{F}^{2}+\frac{\lambda}{2}\left\|\mathbf{D}-\mathbf{D}_{0}\right\|_{F}^{2}+<\mathbf{\Lambda}, \mathbf{D}_{0}-\mathbf{D}_{0}\mathbf{C}>,
\end{equation}
where $\mathbf{\Lambda}$ denotes the Lagrange multiplier and the operator $<\cdot>$ denotes dot product. 

 Letting $\frac{\partial{\mathcal{L}(\mathbf{D}_{0}, \mathbf{C})}}{\partial{\mathbf{D}_{0}}}=0$, it gives that
\begin{equation}
\label{thm2:equ4}
\mathbf{\Lambda}\mathbf{V}_{c}\mathbf{V}_{c}^{T}=\lambda\mathbf{E}.
\end{equation}

Letting $\frac{\partial{\mathcal{L}(\mathbf{D}_{0}, \mathbf{C})}}{\partial{\mathbf{C}}}=0$, it gives that
\begin{equation}
\label{thm2:equ5}
\mathbf{V}_{r}\mathbf{V}_{r}^{T}=\mathbf{V}_{r}\mathbf{\Sigma}_{r}\mathbf{U}_{r}^{T}\mathbf{\Lambda}.
\end{equation}

From (\ref{thm2:equ5}), $\mathbf{\Lambda}$ must be in the form of $\mathbf{\Lambda}=\mathbf{U}_{r}\mathbf{\Sigma}_{r}^{-1}\mathbf{V}_{r}^{T}+\mathbf{U}_{c}\mathbf{M}$ for some $\mathbf{M}$. Substituting $\mathbf{\Lambda}$ into (\ref{thm2:equ4}), it gives that 
\begin{equation}
\label{thm2:equ6}
\mathbf{U}_{c}\mathbf{M}\mathbf{V}_{c}\mathbf{V}_{c}^{T}=\lambda\mathbf{E}.
\end{equation}

Thus, we have $\|\mathbf{E}\|_{F}^{2}=\frac{1}{\lambda^{2}}\|\mathbf{U}_{c}\mathbf{M}\mathbf{V}_{c}\mathbf{V}_{c}^{T}\|=\frac{1}{\lambda^{2}}\|\mathbf{M}\mathbf{V}_{c}\|_{F}^{2}$. Clearly, $\|\mathbf{E}\|_{F}^{2}$ is minimized when $\mathbf{M}\mathbf{V}_{c}$ is a diagonal matrix and can be denoted by $\mathbf{M}\mathbf{V}_{c}=\mathbf{\Sigma}_{c}$, \emph{i.e.}, $\mathbf{E}=\frac{1}{\lambda}\mathbf{U}_{c}\mathbf{\Sigma}_{c}\mathbf{V}_{c}^{T}$. Thus, the SVD of $\mathbf{D}$ could be chosen as 
\begin{equation}
\mathbf{D} = \mathbf{U}\mathbf{\Sigma}\mathbf{V}^{T}=
\left[                 
      \mathbf{U}_{r}\  \mathbf{U}_{c}\\  
\right]     
\left[                 
  \begin{array}{cc}   
      \mathbf{\Sigma}_{r} & \mathbf{0}\\  
      \mathbf{0} &  \frac{1}{\lambda}\mathbf{\Sigma}_{c}\\ 
     \end{array}
\right]  
\left[                 
  \begin{array}{c}   
      \mathbf{V}_{r}^{T}\\  
      \mathbf{V}_{c}^{T}\\ 
     \end{array}
\right]. 
\end{equation}

Thus, the minimal cost of (\ref{thm2:equ2}) is given by 
\begin{equation}
\label{thm2:equ7}
\begin{aligned}
\mathcal{L}_{\min}(\mathbf{D}_{0}^{\ast},\mathbf{C}^{\ast})
&=\frac{1}{2}\|\mathbf{V}_{r}\mathbf{V}_{r}^{T}\|_{F}^{2}+\frac{\lambda}{2}\|\frac{1}{\lambda}\mathbf{\Sigma}_{c }\|_{F}^{2}\\
&=\frac{1}{2} r+\frac{\lambda}{2}\sum_{i=r+1}^{\min\{m,n\}}\sigma_{i}^{2},
\end{aligned}
\end{equation}
where $\sigma_{i}$ is the $i$-th largest singular value of $\mathbf{D}$. Let $k$ be the optimal $r$ to (\ref{thm2:equ7}), then we have $k=\mathrm{argmin}_{r}r+\lambda\sum_{i>r}\sigma_{i}^{2}$. 
\end{proof}

Theorem~\ref{thm2} shows that the skinny SVD of $\mathbf{D}$ is automatically separated into two parts, the top and the bottom one correspond to a desired clean data $\mathbf{D}_{0}$ and the possible corruptions $\mathbf{E}$, respectively. Such a PCA-like result provides a good explanation toward the robustness of our method, \textit{i.e.}, the clean data can be recovered by using the first $k$ leading singular vectors of $\mathbf{D}$. It should be pointed out that the above theoretical results (Lemma 1 and Theorem 1) have been presented in~\cite{Peng2015Connections} for building the connections between Frobenius norm based representation and nuclear norm based representation in theory. Different from \cite{Peng2015Connections}, this paper mainly considers how to utilize this result to achieve robust and automatic subspace learning. 

\begin{figure*}[!t]
\centering{
\subfigure[The coefficient matrix $\mathbf{C}^{\ast}$ obtained by PCE.]{\label{fig1a}\includegraphics[width=0.41\textwidth]{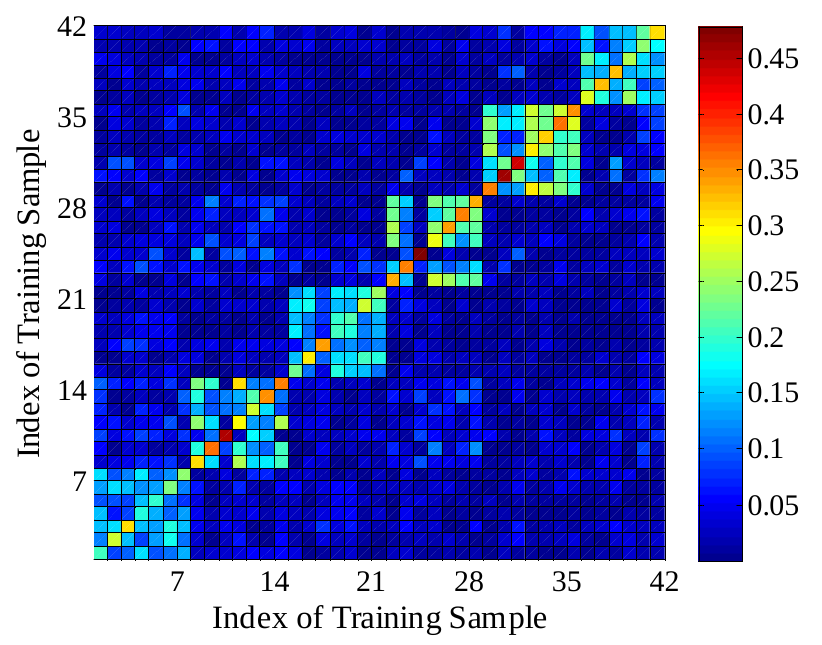}}\hspace{6mm}
\subfigure[Singular values of $\mathbf{C}^{\ast}$]{\label{fig1b}\includegraphics[width=0.43\textwidth]{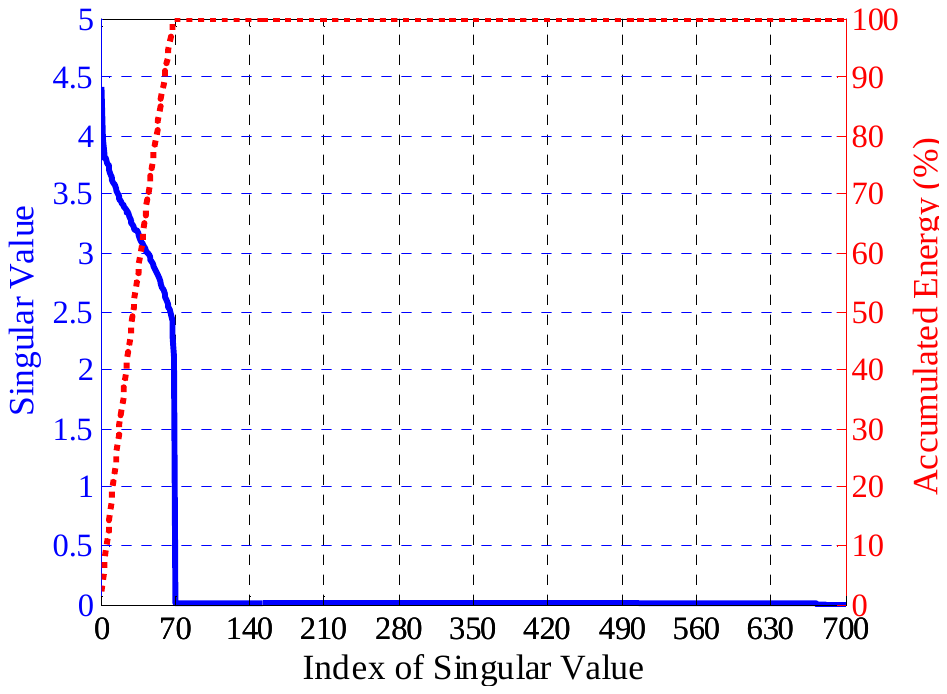}}
}
\caption{\label{fig1} An illustration using 700 AR facial images. (a) PCE can obtain a block-diagonal affinity matrix, which is benefit to classification. For better illustration, we only show the affinity matrix of the data points belonging to the first seven categories. (b) The intrinsic dimension of the used whole data set is exactly 69, \emph{i.e.}, $m^{\prime}=k=69$ for 700 samples. This result is obtained without truncating the trivial singular values like PCA. In the Fig.~\ref{fig1b}, the dotted line denotes the accumulated energy of the first $k$ singular value.}
\end{figure*}

\figurename~\ref{fig1} gives an example to show the effectiveness of PCE. We carried out experiment using 700 clean AR facial images~\cite{Martinez1998} as training data that distribute over 100 individuals. \figurename~\ref{fig1a} shows the coefficient matrix $\mathbf{C}^{\ast}$ obtained by PCE. One can find that the matrix is approximately block-diagonal, \emph{i.e.}, $c_{ij}\ne 0$ if and only if the corresponding points $\mathbf{d}_{i}$ and $\mathbf{d}_{j}$ belong to the same class. Moreover, we perform SVD over $\mathbf{C}^{\ast}$ and show the singular values of $\mathbf{C}^{\ast}$ in~\figurename~\ref{fig1b}. One can find that only the first $69$ singulars values are nonzero. In other words, the intrinsic dimension of the entire data set is $69$ and the first $69$ singular values can preserve 100\% information. It should be pointed out that, PCE does not set a parameter to truncate the trivial singular values like PCA and PCA-like methods~\cite{Yan2007A}, which incorporates all energy into a small number of dimension.

\subsection{Intrinsic Dimension Estimation and Projection Learning}
 
After obtaining the coefficient matrix $\mathbf{C}^{\ast}$, PCE builds a similarity graph and embeds it into an $m^{\prime}$-dimensional space by following NPE~\cite{He2005}, \emph{i.e.},
\begin{equation}
	\label{eq3.3}
	\min_{\mathbf{\Theta}} \frac{1}{2}\|\mathbf{\Theta}^{T}\mathbf{D} - \mathbf{\Theta}^{T}\mathbf{D}\mathbf{A}\|_{F}^{2}, \hspace{3mm} \mathrm{s.t.}\hspace{1mm}\mathbf{\Theta}^{T}\mathbf{D}\mathbf{D}^{T}\mathbf{\Theta}=\mathbf{I},
\end{equation}
where $\mathbf{\Theta}\in \mathds{R}^{m\times m^{\prime}}$ denotes the projection matrix. 

One challenging problem arising in dimension reduction is to determine the value of $m^{\prime}$, most existing methods experimentally set this parameter, which is very computational inefficiency. To solve this problem, we propose estimating the feature dimension using the rank of the affinity matrix $\mathbf{A}$ and have the following theorem:
\begin{theorem}
	\label{thm3} 
	For a given data set $\mathbf{D}$, the feature dimension $m^{\prime}$ is upper bounded by the rank of $\mathbf{C}^{\ast}$, i.e.,
	\begin{equation}
		\label{thm3:equ1}
		m^{\prime} \leq k.
	\end{equation}
\end{theorem}

\begin{proof}
	It is easy to see that Eq.(\ref{eq3.3}) has the following equivalent variation:
	\begin{equation}
		\label{thm3:equ2}
		\mathbf{\Theta}^{\ast}=\mathrm{argmax}\frac{\mathbf{\Theta}^{T}\mathbf{D}(\mathbf{A}+\mathbf{A}^{T}-\mathbf{A}\mathbf{A}^{T})\mathbf{D}^{T}\mathbf{\Theta}}
		{\mathbf{\Theta}^{T}\mathbf{D}\mathbf{D}^{T}\mathbf{\Theta}}.
	\end{equation}	
	
	We can see that the optimal solution to Eq.(\ref{thm3:equ2}) consists of $m^{\prime}$ leading eigenvectors of the following generalized Eigen decomposition problem:
	\begin{equation}
	\label{thm3:equ3}
		\mathbf{D}(\mathbf{A}+\mathbf{A}^{T}-\mathbf{A}\mathbf{A}^{T})\mathbf{D}^{T}\mathbf{\theta}=\sigma\mathbf{D}\mathbf{D}^{T}\mathbf{\theta},
	\end{equation}
where $\sigma$ is the corresponding singular value of the problem. 

As $\mathbf{A}=\mathbf{A}^{T}=\mathbf{A}\mathbf{A}^{T}$, then (\ref{thm3:equ3}) can be rewritten 
\begin{equation}
	\label{thm3:equ4}
	\mathbf{D}\mathbf{A}\mathbf{D}^{T}\mathbf{\theta} =\sigma\mathbf{D}\mathbf{D}^{T} \mathbf{\theta}.
\end{equation}

From Theorem~\ref{thm2}, we have $rank(\mathbf{D})>rank(\mathbf{A})=k$, where $k$ is calculated according to Theorem~\ref{thm2}. Thus, the above generalized Eigen decomposition problem has at most $k$ eigenvalues larger than zeros, \textit{i.e.},  the rank of $\mathbf{\Theta}$ is upperly bounded by $k$. This gives the result. 
\end{proof}


Algorithm~\ref{algorithm1} summarizes the procedure of PCE. Note that, it does not require $\mathbf{A}$ to be a symmetric matrix.


\begin{algorithm}[!h]
    \caption{Automatic Subspace Learning via Principal Coefficients Embedding}
    \label{algorithm1}
    \begin{algorithmic}[1]
    \REQUIRE
    A collection of training data points $\mathbf{D}=\{\mathbf{d}_i\} $ sampled from a union of linear subspaces and the balanced parameter $\lambda>0$. 
        \STATE Perform the full SVD or skinny SVD on $\mathbf{D}$, \emph{i.e.}, $\mathbf{D}=\mathbf{U}\mathbf{\Sigma}\mathbf{V}^{T}$, and get the $\mathbf{C}=\mathbf{V}_{k}\mathbf{V}_{k}^{T}$, where $\mathbf{V}_{k}$ consists of $k$ column vector of $\mathbf{V}$ corresponding to $k$ largest singular values, where $k=\mathrm{argmin}_{r}r+\lambda\sum_{i>r}\sigma_{i}^{2}(\mathbf{D})$ and $\sigma_{i}(\mathbf{D})$ is the $i$-th singular value of $\mathbf{D}$.
    \STATE Construct a similarity graph via $\mathbf{A}=\mathbf{C}$.
    \STATE Embed $\mathbf{A}$ into a $k$-dimensional space and get the projection matrix $\mathbf{\Theta}\in \mathds{R}^{ m\times k}$ that consists of the eigenvectors corresponding to the $k$ largest eigenvalues of the following generalized eigenvector problem Eq.(\ref{thm3:equ2}).
    \ENSURE The projection matrix $\mathbf{\Theta}$. 
    For any data point $\mathbf{y}\in span\{\mathbf{D}\}$, its low-dimensional representation can be obtained by $\mathbf{z}=\mathbf{\Theta}^{T}\mathbf{y}$.
    \end{algorithmic}
\end{algorithm}

\subsection{Computational Complexity Analysis}
\label{sec3.3}

For a training data set $\mathbf{D}\in\mathds{R}^{m\times n}$, PCE performs the skinny SVD over $\mathbf{D}$ in $O(m^{2}n+mn^{2}+n^{3})$. However, a number of fast SVD methods can speed up this procedure. For example, the complexity can be reduced to $O(mnr)$ by Brand's method~\cite{Brand2006}, where $r$ is the rank of $\mathbf{D}$. Moreover, PCE estimates the feature dimension $k$ in $O(rlogr)$ and solves a sparse generalized eigenvector problem in $O(mn+mn^2)$ with Lanczos eigensolver. Putting everything together, the time complexity of PCE is $O(mn+mn^2)$ due to $r\ll \min(m,n)$. 


\section{Experiments and Results}
\label{sec4}

In this section, we reported the performance of PCE and six state-of-the-art unsupervised feature extraction methods including Eigenfaces~\cite{Turk1991}, Locality Preserving Projections (LPP)~\cite{He2003,He2005Lap}, neighbourhood Preserving Embedding (NPE)~\cite{He2005}, L1-graph~\cite{Cheng2010}, Non-negative Matrix Factorization (NMF)~\cite{Hoyer2004,Guan2011:Mani}, RPCA~\cite{Candes2011}, NeNMF~\cite{Guan2012:NeNMF}, and Robust Orthonormal Subspace Learning (ROSL)~\cite{Shu2014:Robust}. Noticed that, NeNMF is one of the most efficient NMF solvers, which can effectively overcome the slow convergence rate, numerical instability and non-convergence issue of NMF. All algorithms are implemented in MATLAB. The used data sets and the codes of our algorithm can be downloaded from the website \textcolor{blue}{\url{http://machineilab.org/users/pengxi}}.

\subsection{Experimental Setting and Data Sets}

We implemented a fast version of L1-graph by using Homotopy algorithm~\cite{Osborne2000} to solve the $\ell_1$-minimization problem. According to~\cite{Yang2010}, Homotopy is one of the most competitive $\ell_1$-optimization algorithms in terms of accuracy, robustness, and convergence speed. For RPCA, we adopted the accelerated proximal gradient method with partial SVD~\cite{Lin2009} which has achieved a good balance between computation speed and reconstruction error. As mentioned above, RPCA cannot obtain the projection matrix for subspace learning. For fair comparison, we incorporated  Eigenfaces with RPCA (denoted by RPCA+PCA) and ROSL (denoted by ROSL+PCA) to obtain the low-dimensional features of the inputs. Unless otherwise specified, we assigned $m^{\prime}=300$ for all the tested methods except PCE which automatically determines the value of $m^{\prime}$. 

In our experiments, we evaluated the performance of these subspace learning algorithms with three classifiers, \emph{i.e.}, sparse representation based classification (SRC)~\cite{Wright2009,Gao2014}, support vector machine (SVM) with linear kernel~\cite{Fan2008}, and the nearest neighbor classifier (NN). For all the evaluated methods, we first identify their optimal parameters using a data partitions and then reported the mean and standard deviation of classification accuracy using 10 randomly sampling data partitions. 

We used eight image data sets including AR facial database~\cite{Martinez1998}, Expended Yale Database B (ExYaleB)~\cite{Georghiades2001}, four sessions of Multiple PIE (MPIE)~\cite{Gross2010}, COIL100 objects database~\cite{Nayar1996}, and the handwritten digital database USPS\footnote{http://archive.ics.uci.edu/ml/datasets.html}. 

The used AR data set contains 2600 samples from 50 male and 50 female subjects, of which 1400 samples are clean images, 600 samples are disguised by sunglasses, and the remaining 600 samples are disguised by scarves. ExYaleB contains 2414 frontal-face images of 38 subjects, and we use the first 58 samples of each subject. MPIE contains the facial images captured in four sessions. In the experiments, all the frontal faces with 14 illuminations\footnote{illuminations: 0,1,3,4,6,7,8,11,13,14,16,17,18,19.} are investigated. For computational efficiency, we downsized all the data sets from the original size to smaller one.~Table~\ref{tab2} provides an overview of the used data sets. 

\begin{table}
\caption{The used databases. $s$ and $n_i$ denote the number of subject and the number of images for each group. }
\label{tab2}
\begin{center}
\begin{small}
\begin{tabular}{lcclc}
\toprule
Databases &  $s$ & $n_i$ & Original Size & Cropped Size\\
\midrule
AR                &    100 & 26 & $165\times 120$ &  $55\times 40$ \\
ExYaleB       &   38 & 58 & $192\times 168$ &  $54\times 48$ \\
MPIE-S1      &    249 & 14 & $100\times 82$ &  $55\times 40$  \\
MPIE-S2      &    203 & 10 & $100\times 82$ &  $55\times 40$  \\
MPIE-S3     &     164 & 10 & $100\times 82$ &  $55\times 40$  \\
MPIE-S4     &     176 & 10 & $100\times 82$ &  $55\times 40$  \\
COIL100  & 100 & 10 & $128\times 128$ & $64\times 64$\\
USPS & 10 & 1100 & $16\times 16$ & -\\
\bottomrule
\end{tabular}
\end{small}
\end{center}
\end{table}

\subsection{The Influence of the Parameter}
\label{sec5.2}

\begin{figure}[!t]
\subfigure [1400 non-disguised  images from the AR database.]{\label{fig3a}\includegraphics[width=0.8\textwidth]{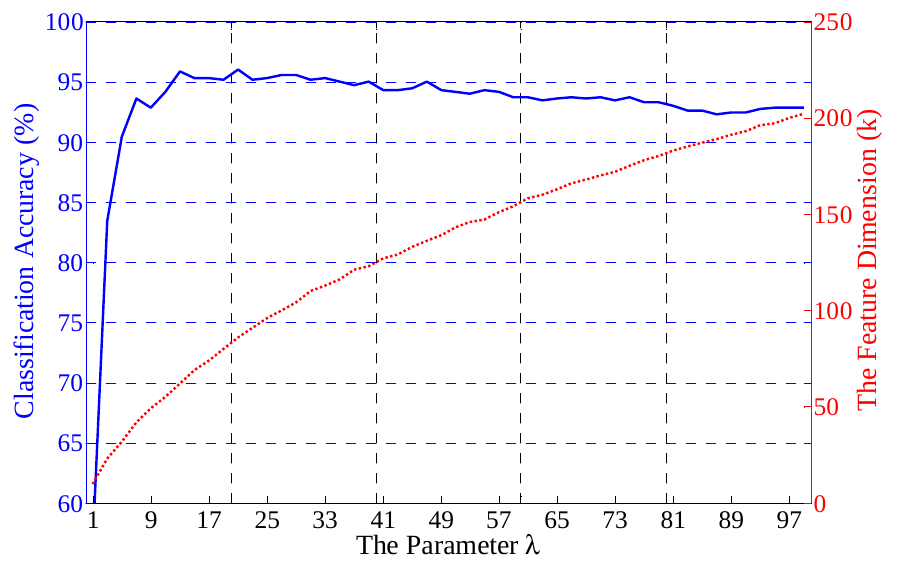}}\\
\subfigure [2204 images from the ExYaleB database.]{\label{fig3b}\includegraphics[width=0.8\textwidth]{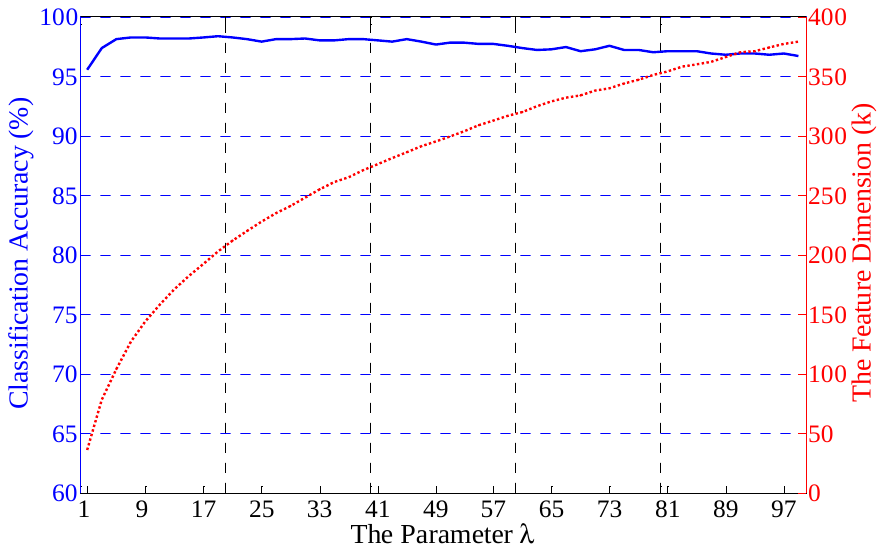}}
\caption{\label{fig3} The influence of the parameter $\lambda$, where the NN classifier is used. The solid and dotted lines denote the classification accuracy and the estimated feature dimension $m^{\prime}$ (\textit{i.e.}, $k$), respectively.}
\end{figure}

In this section, we investigate the influence of parameters of PCE. Besides the aforementioned subspace clustering methods, we also report the performance of CorrEntropy based Sparse Representation (CESR)~\cite{He2011:Maxi} as a baseline. Noticed that, CESR is a not subspace learning method, which performs like SRC to  classify each testing sample by finding which subject produces the minimal reconstruction error. By following the experimental setting in~\cite{He2011:Maxi}, we evaluated CESR using the non-negativity constraint with 0. 

\textbf{The influence of $\lambda$:} PCE uses the parameter $\lambda$ to measure the possible corruptions and estimate the feature dimension $m^{\prime}$. To investigate the influence of $\lambda$ on the classification accuracy and the estimated dimension, we increased the value of $\lambda$ from 1 to 99 with an interval of 2 by performing experiment on a subset of AR database and a subset of Extended Yale Database B. The used data sets include 1400 clean images over 100 individuals and 2204 samples over 38 subjects. In the experiment, we randomly divided each data set into two parts with equal size for training and testing.

\figurename~\ref{fig3} shows that a larger $\lambda$ will lead to a larger $m^{\prime}$ but does not necessarily bring a higher accuracy since the value of $\lambda$ does reflect the errors contained into inputs. For example, while $\lambda$ increases from $13$ to $39$, the recognition accuracy of PCE on AR almost remains unchanged, which ranges from 93.86\% to 95.29\%.

\begin{table*}[!t]
\caption{Performance comparison among different algorithms using  \textbf{ExYaleB}, where training data and testing data consist of 1520 and 684 samples, respectively. PCE, Eigenfaces, and NMF have only one parameter. PCE needs specifying the balanced parameter $\lambda$ but it automatically computes the feature dimension. All methods except PCE extract 300 features for classification. ``Para.'' indicates the tuned parameters. Note that, the second parameter of PCE denotes $m^{\prime}$ (\emph{i.e.}, $k$) which is automatically calculated via Theorem~\ref{thm2}.}
\label{tab3}
\centering
\begin{footnotesize}
\begin{tabular}{l llr| llr| llr}
\toprule
\multicolumn{1}{c}{Classifiers} & \multicolumn{3}{c|}{SRC} & \multicolumn{3}{c|}{SVM} & \multicolumn{3}{c}{NN} \\
\hline
\multicolumn{1}{c}{Algorithms} & Accuracy & Time (s) & Para. & Accuracy & Time (s) & Para. & Accuracy & Time (s) & Para.\\
\midrule
PCE & \textbf{96.90$\pm$0.74}  & 23.50$\pm$2.36  & 5, 118  & \textbf{98.93$\pm$0.18}  & 7.44$\pm$0.37  & 50, 329  & \textbf{97.03$\pm$0.57}  & 6.96$\pm$0.71  & 5, 118 \\
PCE2 & \textbf{96.92$\pm$0.59}  & 28.02$\pm$2.84  & 16.00  & 98.20$\pm$0.43  & 8.07$\pm$0.67  & 26.00  & 96.86$\pm$0.57  & 7.89$\pm$0.88  & 19.00 \\
Eigenfaces & 95.32$\pm$0.80  & 27.79$\pm$0.22  & - & 95.53$\pm$0.85  & 5.65$\pm$0.14  & - & 82.53$\pm$1.70  & 4.97$\pm$0.14  & -\\
LPP & 83.87$\pm$6.59  & 17.20$\pm$0.71  & 9.00  & 87.92$\pm$9.12  & 7.40$\pm$0.12  & 2.00  & 79.97$\pm$1.36  & 7.18$\pm$0.19  & 3.00 \\
NPE & 90.47$\pm$15.72  & 37.80$\pm$0.45  & 50.00  & 82.50$\pm$8.74  & 27.57$\pm$0.24  & 47.00  & 93.35$\pm$0.53  & 28.37$\pm$0.30  & 49.00 \\
L1-graph & 91.29$\pm$0.60  & 633.95$\pm$47.94  & 1e-2,1e-1 & 82.08$\pm$1.66  & 870.04$\pm$61.01  & 1e-3,1e-3 & 89.75$\pm$0.70  & 988.27$\pm$74.98  & 1e-2,1e-3\\
NMF & 87.54$\pm$1.15  & 137.46$\pm$6.26  & - & 91.59$\pm$1.09  & 19.39$\pm$0.36  & - & 72.11$\pm$1.44  & 11.13$\pm$0.03  & -\\
NeNMF & 87.09$\pm$1.11 & 72.10$\pm$6.37 & - & 76.73$\pm$2.14 & 10.38$\pm$0.66 & - & 47.63$\pm$1.03 & 6.89$\pm$0.02 & -\\
RPCA+PCA & 95.88$\pm$0.56  & 497.48$\pm$32.72  & 0.30  & 95.79$\pm$1.02  & 466.17$\pm$35.85  & 0.10  & 82.57$\pm$1.18  & 466.11$\pm$42.70  & 0.20 \\
ROLS+PCA & 95.73$\pm$0.77 & 765.95$\pm$15.95 & 0.23 & 95.03$\pm$0.86 & 733.46$\pm$18.09 & 0.19 & 81.33$\pm$1.65 & 732.61$\pm$15.58 & 0.35\\
\bottomrule
\end{tabular}
\end{footnotesize}
\end{table*}

\textbf{PCE with the fixed $m^{\prime}$:} To further show the effectiveness of our dimension determination method, we investigated the performance of PCE by manually specifying $m^{\prime}=300$, denoted by PCE2. We carried out the experiments on ExYaleB by choosing 40 samples from each subject as training data and using the rests for testing.~Table~\ref{tab3} reports the result from which we can find that:
\begin{itemize}
  \item the automatic version of our method, \emph{i.e.}, PCE, performs competitive to PCE2 which manually set $m^{\prime}=300$. This shows that our dimension estimation method can accurately estimate the feature dimension. 
  \item both PCE and PCE2 outperform the other methods by a considerable performance margin. For example, PCE is 3.68\% at least higher than the second best method when the NN classifier is used.
  \item Although PCE is not the fastest algorithm, it achieves a good balance between recognition rate and computational efficiency. In the experiments, PCE, Eigenfaces, LPP, NPE, and NeNMF are remarkably faster than other baseline methods. Moreover, NeNMF is remarkably faster than NMF while achieving a competitive performance. 
\end{itemize}

\textbf{Tuning $m^{\prime}$ for the baseline methods:} To show the dominance of the dimension estimation of PCE, we reported the performance of all the baseline methods in two settings, \textit{i.e.}, $m^{\prime}=300$ and the optimal $m^{\prime}$. The later setting is achieved by finding an optimal $m^{\prime}$ from 1 to 600 so that the algorithm achieves their highest classification accuracy. We carried out the experiments on ExYaleB by selecting 20 samples from each subject as training data and using the rests for testing. Note that, we only tuned $m^{\prime}$ for the baseline algorithms and PCE automatically identifies this parameter. Table~\ref{tab3b} shows that PCE remarkably outperforms the investigated methods in two settings even though all parameters including $m^{\prime}$ are tuned for achieving the best performance of the baselines. 

\begin{table}[!t]
\caption{Performance comparison among different algorithms using \textbf{ExYaleB)}. Besides $m^{\prime}=300$, all methods except PCE are with the tuned $m^{\prime}$.}
\label{tab3b}
\centering
\begin{footnotesize}
\begin{tabular}{l| ll| lll}
\toprule
\multicolumn{1}{c|}{\multirow{2}{*}{Methods}} & \multicolumn{2}{c|}{Fixed $m^{\prime}$ ($m^{\prime}=300$)} & \multicolumn{3}{c}{The Tuned $m^{\prime}$}  \\
\cline{2-6}
& Accuracy & Para. & Accuracy & Para. & $m^{\prime}$\\
\midrule
PCE & \textbf{93.51} & 263 & \textbf{94.63} & 95 & 162\\
Eigenfaces & 70.71&  & 76.11 & -  & 353  \\
LPP & 77.76 & 3.00& 76.73 & 3  & 312\\
NPE & 85.18 & 43.00& 87.54 & 65  & 405\\
L1-graph & 88.78 & 1e-2,1e-3& 89.18 & 1e-2,1e-2  & 532\\
NMF & 62.21 & -  &  70.15 & -  & 214\\
NeNMF & 51.42 & - & 69.88 & - & 148\\
RPCA+PCA & 70.86 & 0.10 & 76.25 & 0.1  & 375\\
ROLS+PCA & 70.46 & 0.35 & 89.18 & 0.39  & 322\\
CESR & 88.71 & 1e-3,1e-3 & 88.85 & 1e-3,1e-3 & 336\\
\bottomrule
\end{tabular}
\end{footnotesize}
\end{table}

\subsection{Performance with Increasing Training Data and Feature Dimension}

\begin{figure*}[!t]
\subfigure []{\label{fig4a}\includegraphics[width=0.39\textwidth]{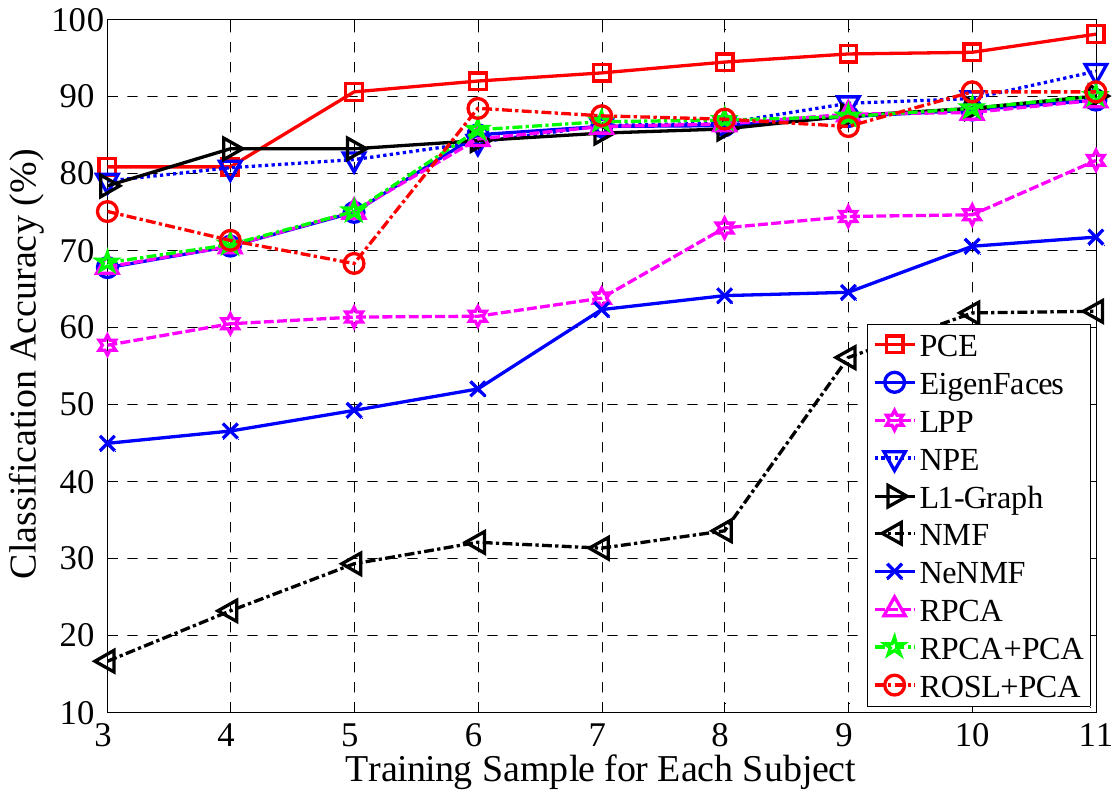}}\hspace{6mm}
\subfigure []{\label{fig4b}\includegraphics[width=0.39\textwidth]{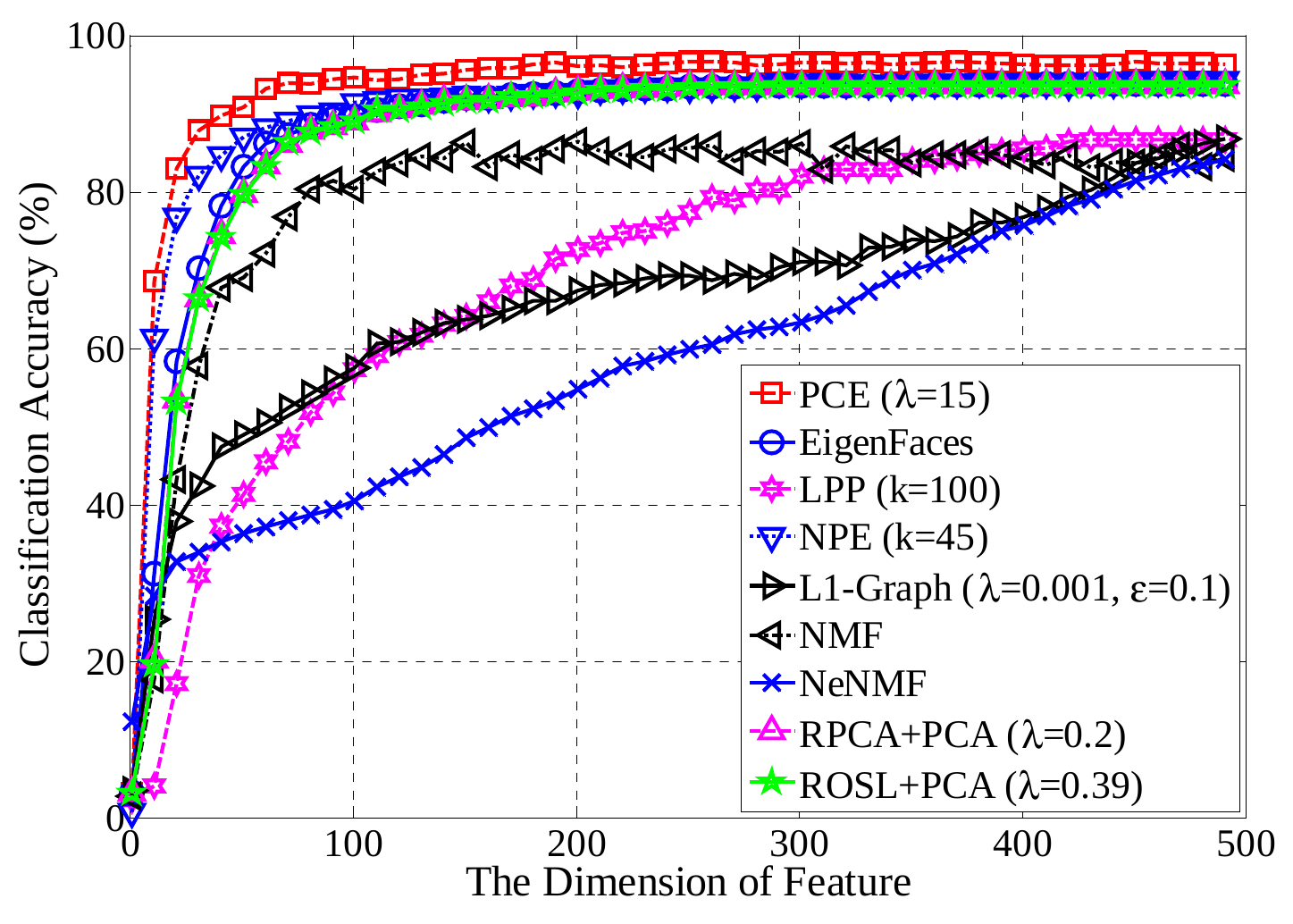}}
\caption{\label{fig4} (a) The performance of the evaluated subspace learning methods with the NN classifier on AR images.  (b) The recognition rates of the NN classifier with different subspace learning methods on ExYaleB. Note that, PCE does not automatically determine the feature dimension in the experiment of performance versus increasing feature dimension. }
\end{figure*}

\begin{table*}[!t]
\caption{Performance comparison among different algorithms using \textbf{the first session of MPIE (MPIE-S1)}. All methods except PCE extract 300 features for classification.}
\label{tab4}
\centering
\begin{footnotesize}
\begin{tabular}{l| llr| llr| llr}
\toprule
\multicolumn{1}{c|}{Classifiers} & \multicolumn{3}{c|}{SRC} & \multicolumn{3}{c|}{SVM} & \multicolumn{3}{c}{NN} \\
\hline
\multicolumn{1}{c|}{Algorithms} & Accuracy & Time (s) & Para. & Accuracy & Time (s) & Para. & Accuracy & Time (s) & Para.\\
\midrule
PCE &\textbf{99.27$\pm$0.32}  & 51.96$\pm$0.29  & 75.00  & \textbf{96.56$\pm$1.23}  & 14.38$\pm$0.52  & 85.00  & \textbf{97.72$\pm$0.55}  & 13.21$\pm$0.59  & 40.00 \\
Eigenfaces & 92.64$\pm$0.56  & 90.64$\pm$0.73  & - & 90.73$\pm$1.81  & 12.87$\pm$0.20  & - & 55.03$\pm$0.93  & 6.21$\pm$0.22  & -\\
LPP & 81.84$\pm$0.94  & 30.58$\pm$2.60  & 10.00  & 70.16$\pm$0.07  & 7.38$\pm$0.54  & 55.00  & 71.31$\pm$2.39  & 4.85$\pm$0.39  & 4.00\\ 
NPE & 80.56$\pm$0.41  & 58.95$\pm$0.77  & 29.00  & 80.25$\pm$0.15  & 36.38$\pm$0.55  & 43.00  & 77.71$\pm$1.65  & 36.19$\pm$0.38  & 49.00\\ 
L1-graph & 80.36$\pm$0.17  & 3856.69$\pm$280.16  & 1e-1,1e-1 & 86.79$\pm$1.62  & 5726.08$\pm$444.82  & 1e-6,1e-5 & 89.82$\pm$1.44  & 8185.55$\pm$503.80  & 1e-6,1e-4\\
NMF & 65.18$\pm$0.87  & 520.94$\pm$6.27  & - & 66.42$\pm$1.66  & 121.89$\pm$0.58  & - & 41.78$\pm$1.18  & 11.03$\pm$0.00  & -\\
NeNMF & 27.42$\pm$1.18  & 277.33$\pm$26.87  & - & 17.57$\pm$1.44  & 36.79$\pm$1.83  & - & 14.66$\pm$1.07  & 10.24$\pm$0.04  & -\\
RPCA+PCA & 92.68$\pm$0.57  & 1755.27$\pm$490.99  & 0.10  & 90.51$\pm$1.26  & 1497.13$\pm$329.00  & 0.30  & 54.95$\pm$1.38  & 1557.33$\pm$358.93  & 0.10 \\
ROLS+PCA & 92.39$\pm$0.91 & 658.29$\pm$39.47 & 0.19 & 89.9$\pm$1.71 & 578.53$\pm$38.61 & 0.35 & 54.82$\pm$1.22 & 593.07$\pm$60.41 & 0.27\\
\bottomrule
\end{tabular}
\end{footnotesize}
\end{table*}

\begin{table*}[!t]
\caption{Performance comparison among different algorithms using \textbf{the second session of MPIE (MPIE-S2)}. All methods except PCE extract 300 features for classification.}
\label{tab5}
\centering
\begin{footnotesize}
\begin{tabular}{l| llr| llr| llr}
\toprule
\multicolumn{1}{c|}{Classifiers} & \multicolumn{3}{c|}{SRC} & \multicolumn{3}{c|}{SVM} & \multicolumn{3}{c}{NN} \\
\hline
\multicolumn{1}{c|}{Algorithms} & Accuracy & Time (s) & Para. & Accuracy & Time (s) & Para. & Accuracy & Time (s) & Para.\\
\midrule
PCE & \textbf{93.87$\pm$0.82}  & 29.00$\pm$0.36  & 90.00  & \textbf{92.63$\pm$0.95}  & 4.97$\pm$0.11  & 40.00  & \textbf{93.18$\pm$0.87}  & 4.14$\pm$0.14  & 75.00 \\
Eigenfaces & 64.36$\pm$2.42  & 81.71$\pm$14.99  & - & 51.72$\pm$2.81  & 0.50$\pm$0.11  & - & 30.86$\pm$1.44  & 0.36$\pm$0.06  & -\\
LPP & 59.62$\pm$2.33  & 36.69$\pm$7.64  & 2.00  & 34.28$\pm$2.53  & 2.73$\pm$0.60  & 2.00  & 62.64$\pm$2.20  & 2.73$\pm$0.84  & 3.00 \\
NPE & 84.65$\pm$0.77  & 33.03$\pm$1.51  & 41.00  & 64.66$\pm$3.03  & 12.45$\pm$0.30  & 27.00  & 85.56$\pm$0.92  & 12.24$\pm$0.24  & 49.00 \\
L1-graph & 47.67$\pm$3.09  & 874.91$\pm$53.69  & 1e-3,1e-3 & 65.41$\pm$1.69  & 657.69$\pm$53.51  & 1e-3,1e-3 & 74.15$\pm$1.67  & 703.54$\pm$37.97  & 1e-2,1e-3\\
NMF & 81.88$\pm$1.31  & 323.93$\pm$8.70  & - & 83.19$\pm$1.47  & 46.72$\pm$1.22  & - & 57.21$\pm$1.38  & 26.01$\pm$0.01  & -\\
NeNMF & 33.38$\pm$1.46  & 128.33$\pm$8.43  & - & 19.65$\pm$1.08  & 26.79$\pm$1.39  & - & 11.94$\pm$0.55  & 14.58$\pm$0.06  & -\\
RPCA+PCA & 91.18$\pm$1.11  & 401.62$\pm$7.46  & 0.20  & 91.18$\pm$1.11  & 401.62$\pm$7.46  & 0.20  & 67.80$\pm$1.93  & 366.50$\pm$8.78  & 0.10 \\
ROLS+PCA & 91.07$\pm$0.91 & 875.79$\pm$74.21 & 0.27 & 83.02$\pm$2.2 & 633.85$\pm$42.38 & 0.43 & 43.7$\pm$0.98 & 761.98$\pm$67.65 & 0.23\\
\bottomrule
\end{tabular}
\end{footnotesize}
\end{table*}

\begin{table*}[!t]
\caption{Performance comparison among different algorithms using \textbf{the third session of MPIE (MPIE-S3)}. All methods except PCE extract 300 features for classification.}
\label{tab6}
\centering
\begin{footnotesize}
\begin{tabular}{l| llr| llr| llr}
\toprule
\multicolumn{1}{c|}{Classifiers} & \multicolumn{3}{c|}{SRC} & \multicolumn{3}{c|}{SVM} & \multicolumn{3}{c}{NN} \\
\hline
\multicolumn{1}{c|}{Algorithms} & Accuracy & Time (s) & Para. & Accuracy & Time (s) & Para. & Accuracy & Time (s) & Para.\\
\midrule
PCE & \textbf{97.79$\pm$0.81}  & 13.14$\pm$0.19  & 75.00  & \textbf{95.37$\pm$1.82}  & 2.74$\pm$0.08  & 65.00  & \textbf{94.04$\pm$0.84}  & 2.29$\pm$0.04  & 65.00 \\
Eigenfaces & 88.04$\pm$0.70 & 29.51$\pm$0.36 & - & 80.99$\pm$2.28 & 2.01$\pm$0.05 & - & 37.96$\pm$1.18 & 0.94$\pm$0.05 & -\\
LPP & 78.73$\pm$2.04 & 28.61$\pm$4.62 & 40.00 & 60.44$\pm$2.49 & 1.61$\pm$0.25 & 3.00 & 65.96$\pm$2.49 & 1.03$\pm$0.13 & 75.00 \\
NPE & 77.83$\pm$3.14 & 25.79$\pm$1.02 & 46.00 & 72.29$\pm$0.99 & 7.56$\pm$0.07 & 7.00 & 79.18$\pm$2.38 & 7.06$\pm$0.09 & 48.00 \\
L1-graph & 70.40$\pm$0.22 & 1315.37$\pm$192.65 & 1e-1,1e-5 & 79.28$\pm$2.54 & 1309.27$\pm$193.38 & 1e-3,1e-3 & 89.40$\pm$2.80 & 1539.26$\pm$226.57 & 1e-3,1e-3\\
NMF & 60.94$\pm$0.80 & 90.64$\pm$0.91 & - & 51.34$\pm$1.68 & 40.04$\pm$0.37 & - & 39.89$\pm$1.04 & 4.28$\pm$0.01 & -\\
NeNMF & 39.90$\pm$1.19 & 61.24$\pm$3.21 & - & 26.66$\pm$1.79 & 20.30$\pm$0.41 & - & 11.93$\pm$0.95 & 3.11$\pm$0.01 & -\\
RPCA+PCA & 88.49$\pm$2.17 & 630.08$\pm$88.89 & 0.10 & 81.02$\pm$2.52 & 491.36$\pm$26.75 & 0.30 & 37.85$\pm$0.83 & 481.87$\pm$25.01 & 0.30 \\
ROLS+PCA & 87.13$\pm$1.53 & 327.14$\pm$30.97 & 0.15 & 78.29$\pm$3.09 & 291.57$\pm$1.64 & 0.47 & 37.23$\pm$1.24 & 297.75$\pm$24.91 & 0.35\\
\bottomrule
\end{tabular}
\end{footnotesize}
\end{table*}

\begin{table*}[!t]
\caption{Performance comparison among different algorithms using \textbf{the fourth session of MPIE (MPIE-S4)}. All methods except PCE extract 300 features for classification.}
\label{tab7}
\centering
\begin{footnotesize}
\begin{tabular}{l| llr| llr| llr}
\toprule
\multicolumn{1}{c|}{Classifiers} & \multicolumn{3}{c|}{SRC} & \multicolumn{3}{c|}{SVM} & \multicolumn{3}{c}{NN} \\
\hline
\multicolumn{1}{c|}{Algorithms} & Accuracy & Time (s) & Para. & Accuracy & Time (s) & Para. & Accuracy & Time (s) & Para.\\
\midrule
PCE & \textbf{98.36$\pm$0.41} & 14.07$\pm$0.31 & 100.00 & \textbf{90.55$\pm$1.02} & 3.04$\pm$0.12 & 45.00 & \textbf{97.34$\pm$0.78} & 2.73$\pm$0.09 & 80.00 \\
Eigenfaces & 92.05$\pm$1.37 & 32.43$\pm$0.32 & - & 82.18$\pm$3.88 & 2.34$\pm$0.05 & - & 43.74$\pm$1.17 & 1.12$\pm$0.05 & -\\
LPP & 64.67$\pm$2.52 & 27.38$\pm$1.57 & 3.00 & 61.47$\pm$1.12 & 1.94$\pm$0.20 & 2.00 & 73.69$\pm$2.68 & 1.11$\pm$0.17 & 2.00 \\
NPE & 84.74$\pm$1.50 & 30.45$\pm$1.28 & 46.00 & 63.80$\pm$1.56 & 9.87$\pm$0.49 & 49.00 & 87.30$\pm$1.10 & 8.54$\pm$0.36 & 45.00 \\
L1-graph & 70.45$\pm$0.31 & 1928.24$\pm$212.21 & 1e-3,1e-3 & 84.67$\pm$2.46 & 1825.09$\pm$197.62 & 1e-3,1e-3 & 93.56$\pm$1.13 & 1767.57$\pm$156.61 & 1e-3,1e-3\\
NMF & 69.41$\pm$1.73 & 98.91$\pm$1.37 & - & 53.48$\pm$2.07 & 47.26$\pm$0.44 & - & 25.47$\pm$1.40 & 4.85$\pm$0.00 & -\\
NeNMF & 40.61$\pm$1.21 & 58.45$\pm$3.76 & - & 23.78$\pm$1.80 & 20.87$\pm$1.17 & - & 14.83$\pm$0.61 & 3.36$\pm$0.02 & -\\
RPCA+PCA & 93.16$\pm$1.17 & 682.27$\pm$39.20 & 0.30 & 84.45$\pm$3.02 & 535.31$\pm$19.08 & 0.10 & 43.66$\pm$0.63 & 514.51$\pm$20.82 & 0.10 \\
ROLS+PCA & 91.8$\pm$1.01 & 200.83$\pm$25.47 & 0.07 & 82.61$\pm$2.12 & 265.63$\pm$7.03 & 0.23 & 43.01$\pm$1.58 & 264.21$\pm$6.46 & 0.27\\
\bottomrule
\end{tabular}
\end{footnotesize}
\end{table*}

\begin{table*}[!t]
\caption{Performance comparison among different algorithms using \textbf{COIL100}. All methods except PCE extract 300 features for classification.}
\label{tab8}
\centering
\begin{footnotesize}
\begin{tabular}{l| llr| llr| llr}
\toprule
\multicolumn{1}{c|}{Classifiers} & \multicolumn{3}{c|}{SRC} & \multicolumn{3}{c|}{SVM} & \multicolumn{3}{c}{NN} \\
\hline
\multicolumn{1}{c|}{Algorithms} & Accuracy & Time (s) & Para. & Accuracy & Time (s) & Para. & Accuracy & Time (s) & Para.\\
\midrule
PCE & \textbf{59.60$\pm$1.94}  & 12.25$\pm$0.32  & 15.00  & \textbf{53.00$\pm$1.22}  & 1.36$\pm$0.01  & 45.00  & \textbf{57.40$\pm$1.83}  & 1.15$\pm$0.03  & 5.00\\ 
Eigenfaces & 57.40$\pm$1.67  & 12.97$\pm$0.25  & - & 44.40$\pm$2.21  & 1.04$\pm$0.06  & - & 54.76$\pm$1.14  & 0.67$\pm$0.06  & -\\
LPP & 45.86$\pm$1.51  & 13.22$\pm$0.54  & 60.00  & 30.20$\pm$3.08  & 0.80$\pm$0.11  & 2.00  & 41.10$\pm$2.15  & 0.63$\pm$0.02  & 90.00 \\
NPE & 47.72$\pm$2.25  & 15.30$\pm$0.28  & 43.00  & 32.78$\pm$2.90  & 5.33$\pm$0.08  & 36.00  & 44.88$\pm$2.12  & 6.81$\pm$0.03  & 49.00 \\
L1-graph & 45.16$\pm$1.83  & 960.80$\pm$123.43  & 1e-2,1e-4 & 39.42$\pm$2.81  & 801.73$\pm$147.83  & 1e-3,1e-3 & 38.06$\pm$1.96  & 664.92$\pm$93.75  & 1e-1,1e-5\\
NMF & 51.42$\pm$2.17  & 76.05$\pm$1.21  & - & 41.74$\pm$2.05  & 32.81$\pm$0.18  & - & 56.82$\pm$1.46  & 6.47$\pm$0.00  & -\\
NeNMF & 57.48$\pm$2.13  & 39.21$\pm$3.18  & - & 35.96$\pm$3.73  & 25.64$\pm$0.52  & - & 59.02$\pm$1.55  & 10.99$\pm$0.01  & -\\
RPCA+PCA & 58.04$\pm$0.90  & 244.92$\pm$50.17  & 0.30  & 45.52$\pm$2.70  & 229.54$\pm$51.06  & 0.20  & 56.48$\pm$1.32  & 227.27$\pm$52.66  & 0.10 \\
ROLS+PCA & 58.11$\pm$1.67 & 447.05$\pm$27.01 & 0.03 & 44.74$\pm$1.69 & 747.66$\pm$98.89 & 0.19 & 57.10$\pm$1.68 & 379.81$\pm$18.42 & 0.03\\
\bottomrule
\end{tabular}
\end{footnotesize}
\end{table*}


In this section, we examined the performance of PCE with increasing training samples and increasing feature dimension. In the first test, we randomly sampled $n_i$ clean AR images from each subject for training and used the rest for testing. Besides the result of RPCA+PCA, we also reported the performance of RPCA without dimension reduction. 

In the second test, we randomly chose a half of images from ExYaleB for training and used the rest for testing. We reported the recognition rate of the NN classifier with the first $m^{\prime}$ features extracted by all the tested subspace learning methods, where $m^{\prime}$ increases from $1$ to $600$ with an interval of 10. From~\figurename~\ref{fig4}, we can conclude: 
\begin{itemize}
  \item PCE performs well even though only a few of training samples are available. Its accuracy is about $90\%$ when $n_i=5$, whereas the second best method achieves the same accuracy when $n_i=9$.
  \item RPCA and RPCA+PCA perform very close, however, RPCA+PCA is more efficient than RPCA.
  \item \figurename~\ref{fig4b} shows that PCE consistently outperforms the other methods. This benefits an advantage of PCE, \emph{i.e.}, PCE obtains a more compact representation which can use a few of variables to represent the entire data. 
\end{itemize}

\subsection{Subspace Learning on Clean Images}

In this section, we performed the experiments using MPIE  and COIL100. For each data set, we split it into two parts with equal size. As did in the above experiments, we set $m^{\prime}=300$ for all the tested methods except PCE. Tables~\ref{tab4}--\ref{tab8} report the results, from which one can find that:

\begin{itemize}
  \item with three classifiers, PCE outperforms the other investigated approaches on these five data sets by a considerable performance margin. For example, the recognition rates of PCE with these three classifiers are 6.59\%, 5.83\%, and 7.90\% at least higher than the rates of the second best subspace learning method on MPIE-S1.
  \item PCE is more stable than other tested methods. Although SRC generally outperforms SVM and NN with the same feature, such superiority is not distinct for PCE. For example, SRC gives an accuracy improvement of 1.02\% over NN to PCE on MPIE-S4. However, the corresponding improvement to RPCA+PCA is about 49.50\%.
  \item PCE achieves the best results in all the tests, while using the least time to perform dimension reduction and classification. PCE, Eigenfaces, LPP, NPE, and NeNMF are remarkably efficient than L1-graph, NMF, RPCA+PCA, and ROSL+PCA.
\end{itemize}

\subsection{Subspace Learning on Corrupted Facial Images}

In this section, we investigated the robustness of PCE against two corruptions using ExYaleB and the NN classifier. The corruptions include the white Gaussian noise (additive noise) and the random pixel corruption (non-additive noise)~\cite{Wright2009}. 

In our experiments, we use a half of images (29 images per subject) to corrupt using these two noises. Specifically, we added white Gaussian noise into the sampled data $\mathbf{d}$ via $\mathbf{\tilde{d}} = \mathbf{d}+\rho \mathbf{n}$, where $\mathbf{\tilde{d}}\in[0\ 255]$, $\rho$ is the corruption ratio, and $\mathbf{n}$ is the noise following the standard normal distribution. For random pixel corruption, we replaced the value of a percentage of pixels randomly selected from the image with the values following a uniform distribution over $[0,\ p_{max}]$, where $p_{max}$ is the largest pixel value of $\mathbf{d}$. After adding the noises into the images, we randomly divide the data into training and testing sets. In other words, both training data  and testing data probably contains corruptions.~\figurename~\ref{fig5} illustrates some results achieved by our method. We can see that PCE successfully identifies the noises from the corrupted samples and recovers the clean data.~Table~\ref{tab9} reports the comparison from which we can see that: 
\begin{itemize}
  \item PCE is more robust than the other tested approaches. When $10\%$ pixels are  randomly corrupted, the accuracy of PCE is at least $9.46\%$ higher than that of the other methods.
  \item with the increase of level of noise, the dominance of PCE is further strengthen. For example, the improvement in accuracy of PCE increases from $9.46\%$ to $23.23\%$ when $\rho$ increases to 30\%.
\end{itemize}

\begin{figure}[!t]
\subfigure [corruption ratio: 10\%]{\label{fig5a}\includegraphics[width=0.46\textwidth]{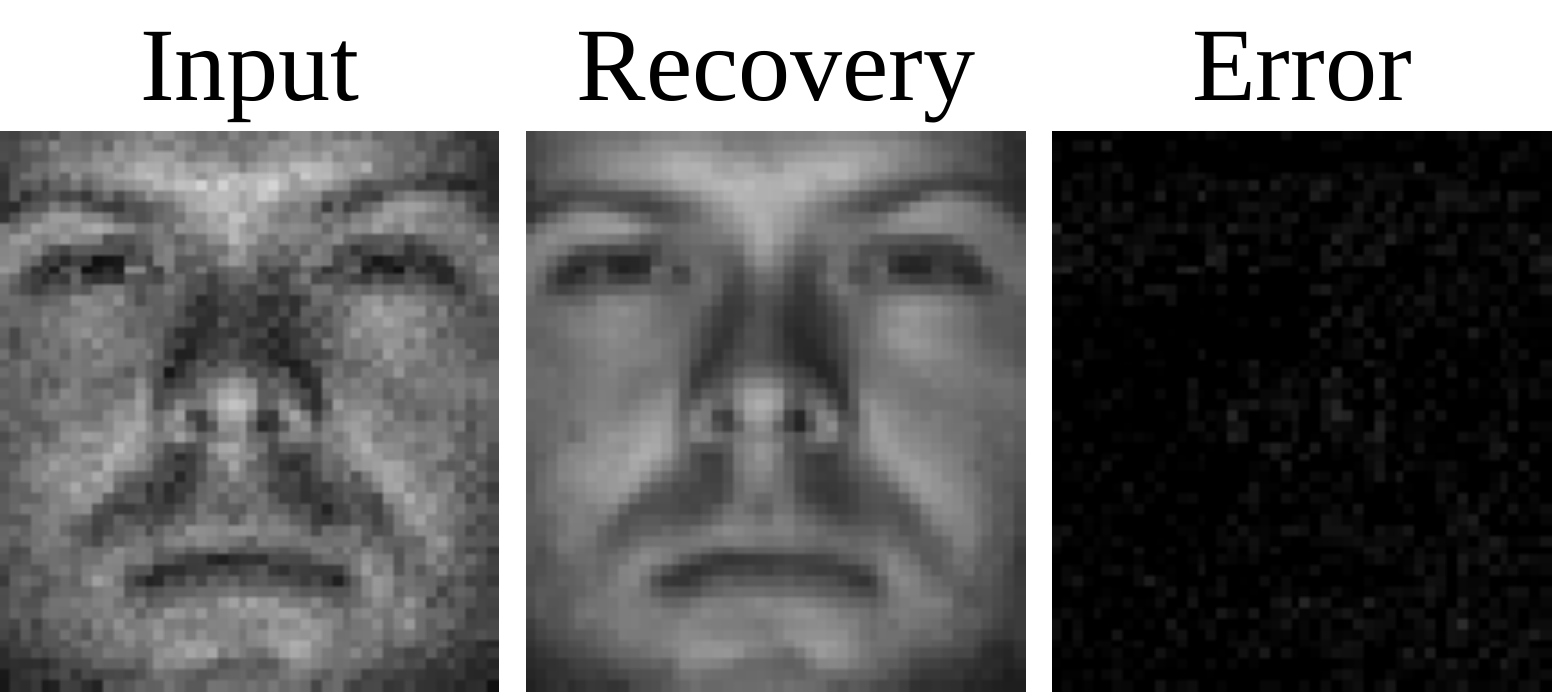}}\hspace{3mm}
\subfigure [corruption ratio: 30\%]{\label{fig5b}\includegraphics[width=0.46\textwidth]{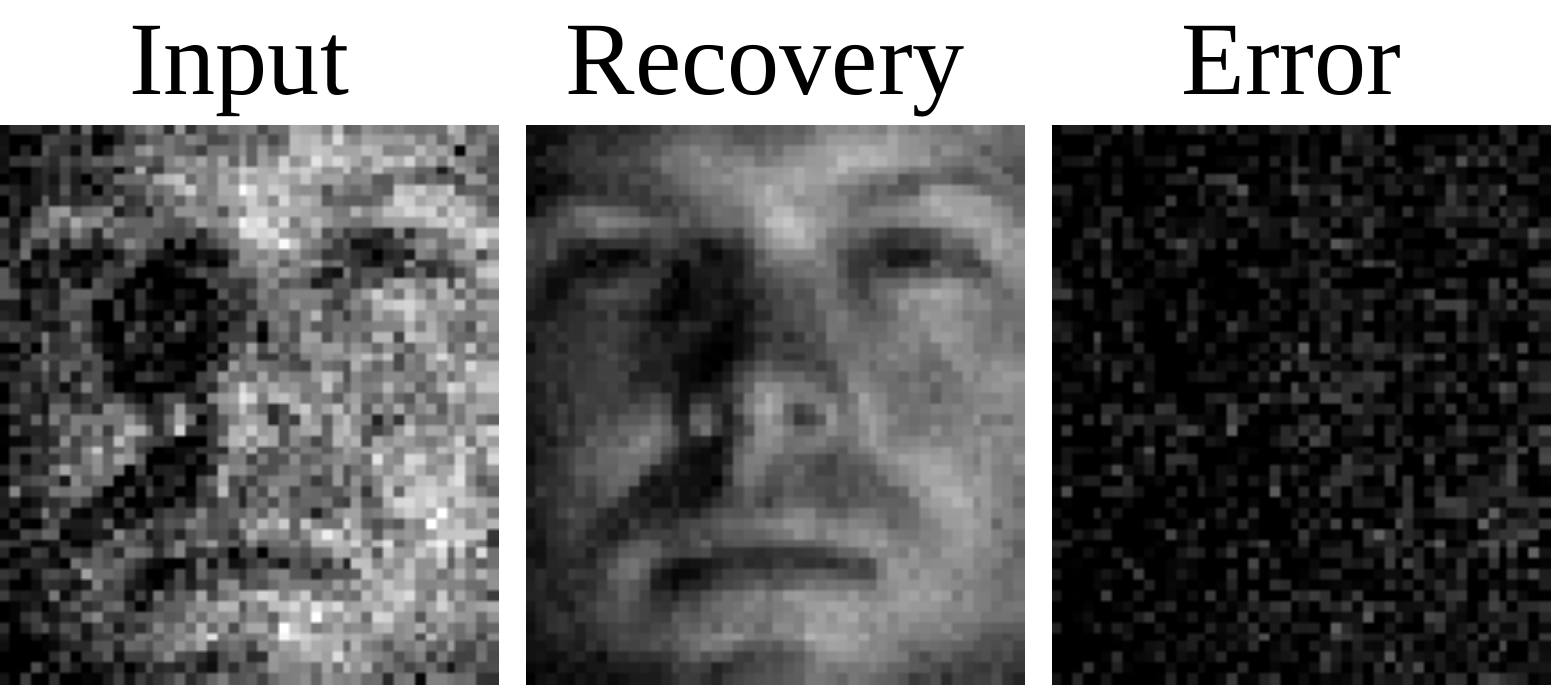}}
\caption{\label{fig5} Some results achieved by PCE over the corrupted ExYaleB data set which is corrupted by the Gaussian noise. The recovery and the error are identified by PCE according to Theorem 1.}
\end{figure}

\begin{table*}[!t]
\caption{Performance of different subspace learning algorithms with the NN classifier using \textbf{the corrupted ExYaleB}. All methods except PCE extract 300 features for classification. RPC is the short for Random Pixel Corruption. The number in the parentheses denotes the level of corruption.}
\label{tab9}
\centering
\begin{footnotesize}
\begin{tabular}{l| lr| lr| lr| lr}
\toprule
\multicolumn{1}{c|}{Corruptions} & \multicolumn{2}{c|}{Gaussian (10\%)} & \multicolumn{2}{c|}{Gaussian (30\%)} & \multicolumn{2}{c|}{RPC (10\%)} & \multicolumn{2}{c}{RPC (30\%)}\\
\hline
\multicolumn{1}{c|}{Algorithms} & Accuracy & Para. & Accuracy & Para.& Accuracy & Para.& Accuracy & Para.\\
\midrule
PCE & \textbf{95.05$\pm$0.63}  & 10.00  & \textbf{93.18$\pm$0.87}  & 5.00  & \textbf{90.12$\pm$0.98}  & 5.00  & \textbf{83.48$\pm$1.04}  & 10.00 \\
Eigenfaces & 41.69$\pm$2.01  & - & 30.86$\pm$1.44  & - & 30.35$\pm$2.05  & - & 25.37$\pm$1.56  & -\\
LPP & 76.94$\pm$0.75  & 2.00  & 62.64$\pm$2.20  & 3.00  & 55.86$\pm$1.27  & 2.00  & 42.76$\pm$1.53  & 2.00 \\
NPE & 91.54$\pm$0.76  & 49.00  & 85.56$\pm$0.92  & 49.00  & 80.66$\pm$0.86  & 49.00  & 60.25$\pm$1.64  & 43.00 \\
L1-graph & 87.36$\pm$0.81  & 1e-3,1e-4 & 74.15$\pm$1.67  & 1e-2,1e-3 & 71.63$\pm$0.90  & 1e-3,1e-4 & 55.02$\pm$2.07  & 1e-4,1e-4\\
NMF & 67.42$\pm$1.41  & - & 57.21$\pm$1.38  & - & 60.57$\pm$1.88  & - & 46.13$\pm$1.41  & -\\
NeNMF & 44.75$\pm$1.28  & - & 42.48$\pm$0.68  & - & 43.27$\pm$1.01  & - & 25.44$\pm$1.21  & -\\
RPCA+PCA & 76.26$\pm$1.12  & 0.20  & 67.80$\pm$1.93  & 0.10  & 64.56$\pm$0.67  & 0.10  & 52.12$\pm$1.34  & 0.10 \\
ROSL+PCA & 76.52$\pm$0.83 & 0.27 & 66.50$\pm$1.49  & 0.27 & 75.94$\pm$1.44 & 0.07 & 65.76$\pm$0.98 &  0.07 \\
\bottomrule
\end{tabular}
\end{footnotesize}
\end{table*}

\subsection{Subspace Learning on Disguised Facial Images}

Besides the above tests on the robustness to corruptions, we also investigated the robustness to real disguises.~Tables~\ref{tab10}--\ref{tab11} reports results on two subsets of AR database. The first subset contains 600 clean images and 600 images disguised with sunglasses (occlusion rate is about $20\%$), and the second one includes 600 clean images and 600 images disguised by scarves (occlusion rate is about $40\%$). Like the above experiment, both training data and testing data will contains the disguised images. From the results, one can conclude that:
\begin{itemize}
  \item PCE significantly outperforms the other tested methods. When the images are disguised by sunglasses, the recognition rates of PCE with SRC, SVM, and NN are $5.88\%$, $23.03\%$, and $11.75\%$ higher than the best baseline method. With respect to the images with scarves, the corresponding improvements are $12.17\%$, $21.30\%$, and $17.64\%$.
  \item PCE is one of the most computationally efficient methods. When SRC is used, PCE is 2.27 times faster than NPE and 497.16 times faster than L1-graph on the faces with sunglasses. When the faces are disguised by scarves, the corresponding speedup are 2.17 and 484.94 times, respectively.
\end{itemize}

\begin{table*}[!t]
\caption{Performance comparison among different algorithms using \textbf{the AR images disguised by sunglasses}. All methods except PCE extract 300 features for classification.}
\label{tab10}
\centering
\begin{footnotesize}
\begin{tabular}{l| llr| llr| llr}
\toprule
\multicolumn{1}{c|}{Classifiers} & \multicolumn{3}{c|}{SRC} & \multicolumn{3}{c|}{SVM} & \multicolumn{3}{c}{NN} \\
\hline
\multicolumn{1}{c|}{Algorithms} & Accuracy & Time (s) & Para. & Accuracy & Time (s) & Para. & Accuracy & Time (s) & Para.\\
\midrule
PCE & \textbf{83.88$\pm$1.38}  & 8.73$\pm$0.90  & 65.00  & \textbf{87.80$\pm$1.57}  & 0.90$\pm$0.10  & 65.00  & \textbf{68.58$\pm$1.96}  & 0.71$\pm$0.11  & 55.00 \\
Eigenfaces & 72.87$\pm$1.99  & 45.48$\pm$5.18  & - & 64.77$\pm$2.96  & 1.62$\pm$0.40  & - & 36.42$\pm$1.69  & 0.78$\pm$0.19  & -\\
LPP & 51.73$\pm$2.77  & 44.20$\pm$8.25  & 95.00  & 44.88$\pm$1.93  & 1.60$\pm$0.57  & 2.00  & 37.37$\pm$2.19  & 1.12$\pm$0.30  & 85.00 \\
NPE & 78.00$\pm$2.27  & 19.84$\pm$0.51  & 47.00  & 49.17$\pm$3.33  & 4.30$\pm$0.04  & 47.00  & 56.83$\pm$1.83  & 4.16$\pm$0.04  & 49.00 \\
L1-graph & 52.00$\pm$1.42  & 4340.22$\pm$573.64  & 1e-4,1e-4 & 48.53$\pm$2.06  & 3899.81$\pm$487.89  & 1e-4,1e-4 & 49.28$\pm$2.68  & 4189.73$\pm$431.98  & 1e-4,1e-4\\
NMF & 47.87$\pm$2.64  & 108.46$\pm$2.98  & - & 43.05$\pm$2.39  & 24.34$\pm$0.81  & - & 31.35$\pm$2.04  & 8.01$\pm$0.01  & -\\
NeNMF & 37.57$\pm$2.41  & 55.36$\pm$4.31  & - & 26.17$\pm$2.30  & 15.14$\pm$0.34  & - & 26.52$\pm$1.43  & 7.62$\pm$0.01  & -\\
RPCA+PCA & 72.07$\pm$2.30  & 1227.08$\pm$519.27  & 0.10  & 63.70$\pm$3.74  & 1044.46$\pm$462.33  & 0.20  & 36.93$\pm$0.90  & 965.76$\pm$385.19  & 0.10 \\
ROLS+PCA & 71.42$\pm$1.46  & 1313.65$\pm$501.05  & 0.31  & 62.67$\pm$3.27  & 1336.00$\pm$549.66  & 0.43  & 35.58$\pm$2.14  & 1245.63$\pm$406.74  & 0.19 \\
\bottomrule
\end{tabular}
\end{footnotesize}
\end{table*}

\begin{table*}[!t]
\caption{Performance comparison among different algorithms using \textbf{the AR images disguised by scarves}. All methods except   PCE extract 300 features for classification.}
\label{tab11}
\centering
\begin{footnotesize}
\begin{tabular}{l| llr| llr| llr}
\toprule
\multicolumn{1}{c|}{Classifiers} & \multicolumn{3}{c|}{SRC} & \multicolumn{3}{c|}{SVM} & \multicolumn{3}{c}{NN} \\
\hline
\multicolumn{1}{c|}{Algorithms} & Accuracy & Time (s) & Para. & Accuracy & Time (s) & Para. & Accuracy & Time (s) & Para.\\
\midrule
PCE & \textbf{83.57$\pm$1.16}  & 8.95$\pm$0.96  & 50.00  & \textbf{87.70$\pm$1.62}  & 0.90$\pm$0.11  & 95.00  & \textbf{66.92$\pm$1.75}  & 0.69$\pm$0.10  & 50.00 \\
Eigenfaces & 69.32$\pm$2.58  & 36.18$\pm$3.97  & - & 63.93$\pm$2.84  & 1.42$\pm$0.34  & - & 30.58$\pm$1.27  & 0.71$\pm$0.19  & -\\
LPP & 49.48$\pm$1.79  & 30.80$\pm$6.29  & 2.00  & 43.07$\pm$1.80  & 1.34$\pm$0.39  & 2.00  & 33.70$\pm$1.70  & 0.85$\pm$0.21  & 90.00 \\
NPE & 62.75$\pm$2.16  & 19.44$\pm$0.62  & 47.00  & 58.23$\pm$2.75  & 4.40$\pm$0.03  & 49.00  & 54.33$\pm$2.37  & 4.09$\pm$0.03  & 49.00 \\
L1-graph & 49.65$\pm$1.42  & 4340.22$\pm$453.64  & 1e-3,1e-4 & 48.53$\pm$2.06  & 5381.96$\pm$467.89  & 1e-4,1e-4 & 49.28$\pm$2.68  & 5189.73$\pm$411.98  & 1e-4,1e-4\\
NMF & 47.17$\pm$2.18  & 109.33$\pm$2.22  & - & 40.55$\pm$2.20  & 23.67$\pm$0.92  & - & 24.58$\pm$1.88  & 5.82$\pm$0.01  & -\\
NeNMF & 32.58$\pm$1.76  & 52.20$\pm$1.81  & - & 21.98$\pm$2.10  & 14.71$\pm$0.08  & - & 18.52$\pm$1.46  & 3.49$\pm$0.00  & -\\
RPCA+PCA & 71.40$\pm$2.67  & 241.45$\pm$4.39  & 0.10  & 66.40$\pm$2.62  & 193.67$\pm$7.76  & 0.10  & 32.27$\pm$1.46  & 195.21$\pm$8.01  & 0.20 \\
ROLS+PCA & 68.57$\pm$1.37  & 557.54$\pm$92.71  & 0.15  & 60.38$\pm$2.53  & 441.79$\pm$95.61  & 0.23  & 31.03$\pm$1.05  & 432.02$\pm$95.28  & 0.07 \\
\bottomrule
\end{tabular}
\end{footnotesize}
\end{table*}

\subsection{Comparisons with Some Dimension Estimation Techniques}

\begin{table}[!t]
\caption{Performance of different dimension estimators with the NN classifier, where $m^{\prime}$ denotes the estimated feature dimension and only the time cost (second) for dimension estimation is taken into consideration.}
\label{tab12}
\centering
\begin{footnotesize}
\begin{tabular}{l| rrrr}
\toprule
Methods & Accuracy & Time Cost & $m^{\prime}$ & Para.\\
\midrule
PCE                   & \textbf{90.31} & 0.57  & 109 & 30	\\
MLE+PCA          & 68.29 & 3.34  & 11.6 & 10	\\
MiND-ML+PCA  & 67.14 & 3.95  & 11 & 10	\\
MiND-KL+PCA  & 72.71 & 2791.19  & 16 & 22	\\
DANCo+PCA    & 71.71 & 28804.01  & 15 & 22	\\
\bottomrule
\end{tabular}
\end{footnotesize}
\end{table}

In this section, we compare PCE with three dimension estimators, \textit{i.e.}, maximum likelihood estimation (MLE)~\cite{Levina2004:Maximum}, minimum neighbor distance Estimators (MiND)~\cite{Lombardi2011:Mini}, and DANCo~\cite{Ceruti2014:DANCo}.  MiND has two variants which are denoted as MiND-ML and MiND-KL. All these estimators need specifying the size of neighborhood of which the optimal value is found from the range of [10 30] with an interval of 2. Since these estimators cannot be used for dimension reduction, we report the performance of these estimators with PCA, \textit{i.e.}, we first estimate the feature dimension with an estimator and then extract features using PCA with the estimated dimension. We carry out experiments with the NN classifier on a subset of the AR data set of which both the training and testing set include 700 non-disguised facial images. Table~\ref{tab12} shows that our approach outperforms the baseline estimators by a considerable performance margin in terms of classification accuracy and time cost.

\subsection{Scalability Evaluation}

In this section, we investigate the scalability performance of PCE by using the whole USPS data set, where $\lambda$ of PCE is fixed as $0.05$. In the experiments, we randomly split the whole data set into two partitions for training and testing, where the number of training samples increases from 500 to 9,500 with an interval of 500 and thus 19 partitions are obtained. \figurename~\ref{fig6} reports the classification accuracy and the time cost taken by PCE. From the results, we could see that the recognition rate of PCE almost remains unchanged when 1,500 samples are available for training. Considering different classifiers, SRC slightly performs better than NN, and both of them remarkably outperform SVM. PCE is computational efficient, it only take about seven seconds to handle 9,500 samples. Moreover, PCE could be further speeded up by adopting large scale SVD methods. However, this has been out of scope for this paper. 

\begin{figure*}[!t]
\subfigure [Classification Accuracy]{\label{fig6a}\includegraphics[width=0.41\textwidth]{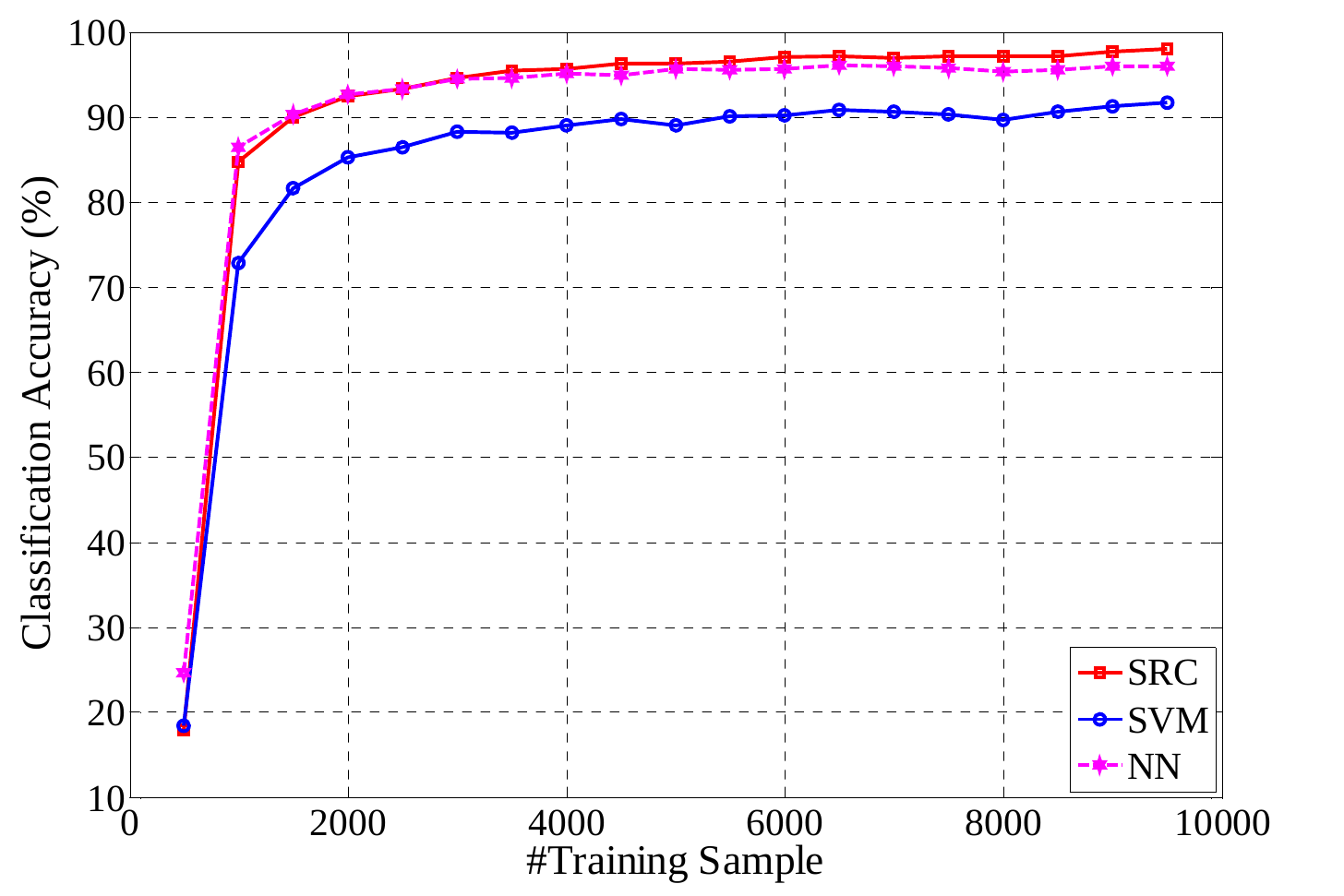}}\hspace{6mm}
\subfigure [Time Cost (Seconds)]{\label{fig6b}\includegraphics[width=0.40\textwidth]{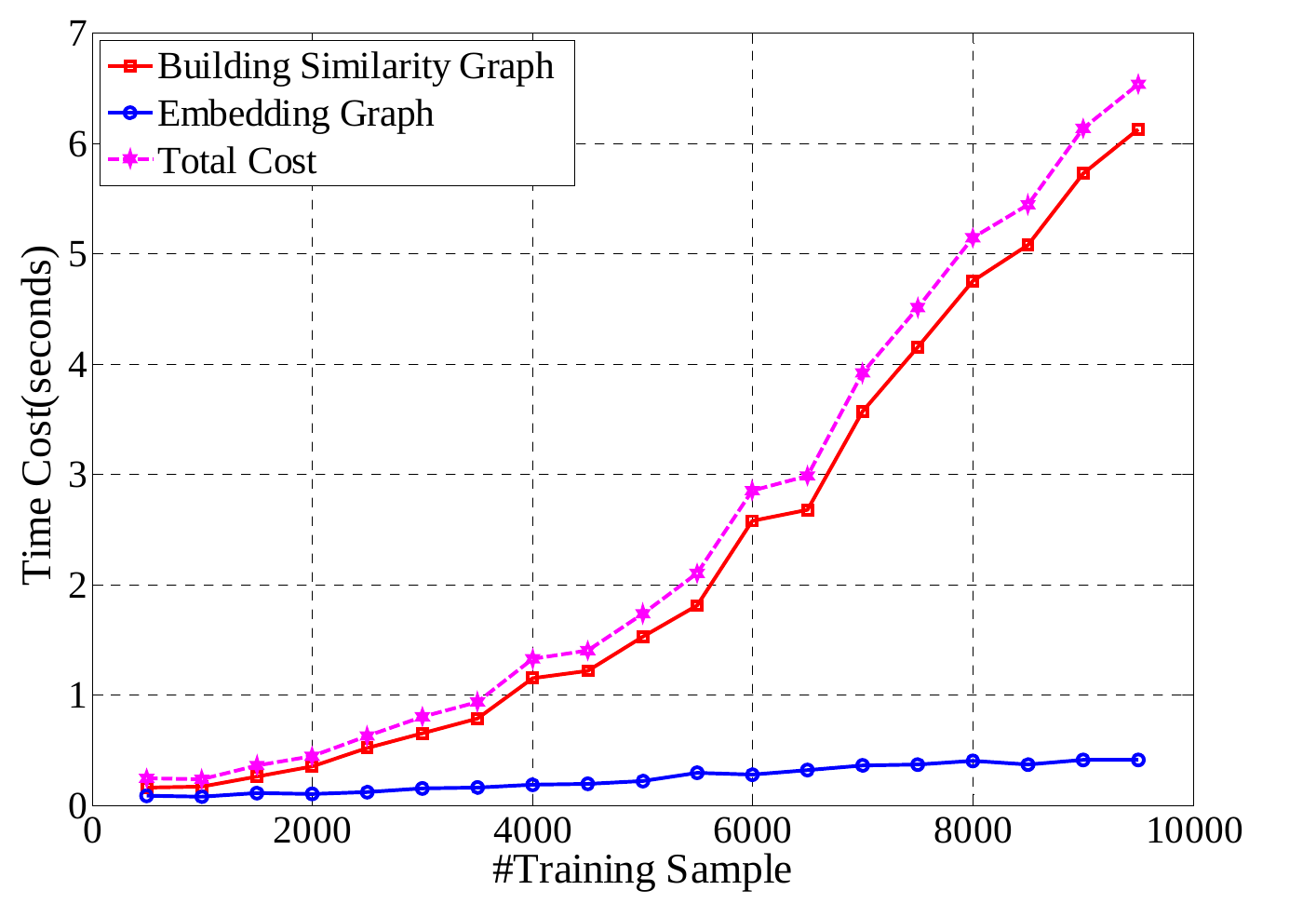}}
\caption{\label{fig6} Scalability performance of PCE on the whole USPS data set, where the number training samples increase from 500 to 9500 with an interval of 500. (a) The recognition rate of PCE with three classifiers. (b) The time costs for different steps of PCE, where \textit{Total Cost} is the cost for building similarity graph and embedding graph. }
\end{figure*}

\section{Conclusion}
\label{sec5}

In this paper, we have proposed a novel unsupervised subspace learning method, called principal coefficients embedding (PCE). Unlike existing subspace learning methods, PCE can automatically determine the optimal dimension of feature space and obtain the low-dimensional representation of a given data set. Experimental results on several popular image databases have shown that our PCE achieves a good performance with respect to additive noise, non-additive noise, and partial disguised images.

The work would be further extended or improved from the following aspects. First, the paper currently only considers one category of image recognition, \emph{i.e.}, image identification. In the future, PCE can be extended to handle the other category of image recognition, \emph{i.e.}, face verification which aims to determine whether a given pair of facial images is from the same subject or not. Second, PCE is a unsupervised method which does not adopt the label information. If such information is available, one can develop the supervised or semi-supervised version of PCE under the framework of graph embedding. Third, PCE can be extended to handle outliers by enforcing $\ell_{2,1}$-norm or Laplacian noises by enforcing $\ell_1$-norm over the errors term in our objective function.

\section*{Acknowledgment}
The authors would like to thank the anonymous reviewers for their valuable comments and suggestions that significantly improve the quality of this paper.

 \bibliography{PCE}
 \bibliographystyle{IEEEtran}

%

%
%
%




\begin{IEEEbiography}[{\includegraphics[width=1in,height=1.25in,clip,keepaspectratio]{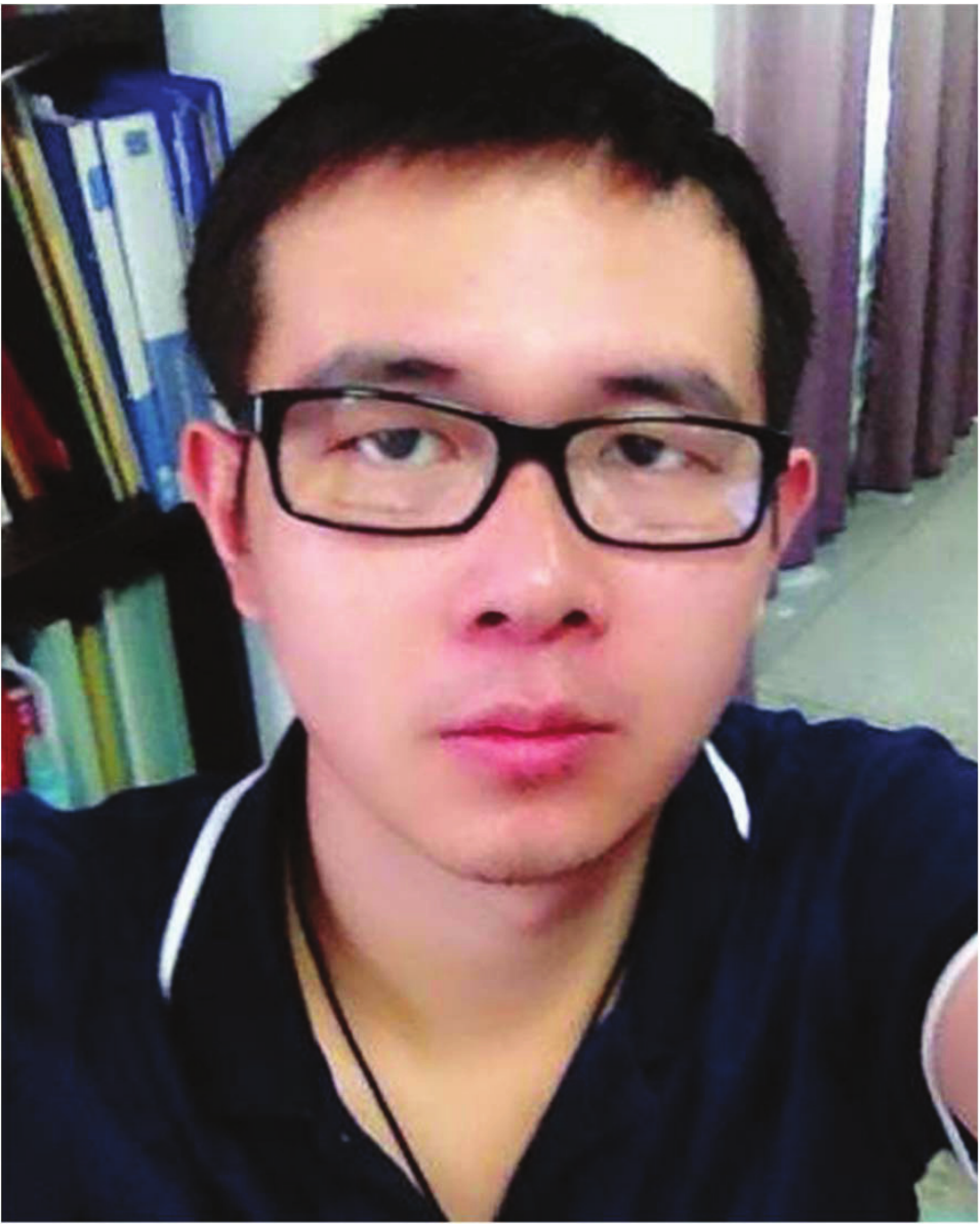}}]{Xi Peng} received the BEng degree in Electronic Engineering and MEng degree in Computer Science from Chongqing University of Posts and Telecommunications, and the Ph.D. degree from Sichuan University, China, respectively. Currently, he is a research scientist at Institute for Infocomm, Research Agency for Science, Technology and Research (A*STAR) Singapore. His current research interests include machine intelligence and computer vision. 

 Dr. Peng is the recipient of Excellent Graduate Student of Sichuan Province, China National Graduate Scholarship, CSC-IBM Scholarship for Outstanding Chinese Students, and Excellent Student Paper of IEEE CHENGDU Section. He has served as a guest editor of \textit{Image and Vision Computing}, a PC member/reviewer for 10 international conferences such as \textit{AAAI Conference on Artificial Intelligence} and a reviewer for over 10 international journals.
\end{IEEEbiography}

\begin{IEEEbiography}[{\includegraphics[width=1in,height=1.25in,clip,keepaspectratio]{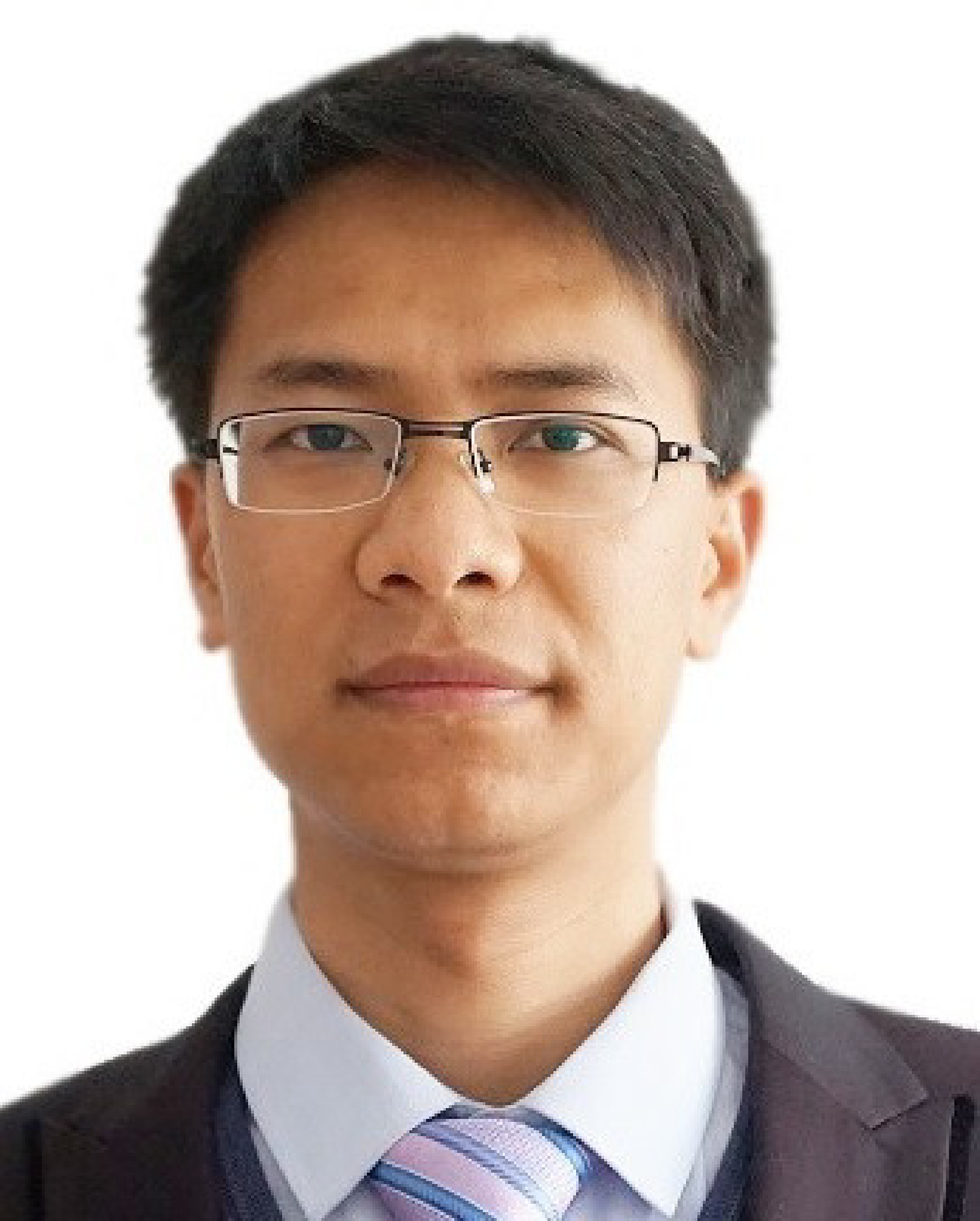}}]{Jiwen Lu} (M'11--SM'15) received the B.Eng. degree in mechanical engineering and the M.Eng. degree in electrical engineering from the Xi'an University of Technology, Xi'an, China, and the Ph.D. degree in electrical engineering from the Nanyang Technological University, Singapore, in 2003, 2006, and 2011, respectively. He is currently an Associate Professor with the Department of Automation, Tsinghua University, Beijing, China. From March 2011 to November 2015, he was a Research Scientist with the Advanced Digital Sciences Center, Singapore. His current research interests include computer vision, pattern recognition, and machine learning. He has authored/co-authored over 130 scientific papers in these areas, where more than 50 papers are published in the IEEE Transactions journals and top-tier computer vision conferences. He serves/has served as an Associate Editor of \textit{Pattern Recognition Letters}, \textit{Neurocomputing}, \textit{Journal of Signal Processing Systems}, t\textit{he IEEE Access} and \textit{the IEEE Biometrics Council Newsletters}, a Guest Editor of \textit{Pattern Recognition}, \textit{Computer Vision and Image Understanding}, \textit{Image and Vision Computing} and \textit{Neurocomputing}, and an elected member of the Information Forensics and Security Technical Committee of the IEEE Signal Processing Society. He is/was an Area Chair for VCIP'16, WACV'16, ICB’16, ICME'15, and ICB'15, a Workshop Co-Chair for ACCV’2016, and a Special Session Co-Chair for VCIP'15. He has given tutorials at several international conferences including CVPR'15, FG'15, ACCV'14, ICME'14, and IJCB'14. He was a recipient of the First-Prize National Scholarship and the National Outstanding Student Award from the Ministry of Education of China in 2002 and 2003, the Best Student Paper Award from Pattern Recognition and Machine Intelligence Association of Singapore in 2012, the Top 10\% Best Paper Award from IEEE International Workshop on Multimedia Signal Processing in 2014, and the National 1000 Young Talents Plan Program in 2015, respectively.
\end{IEEEbiography}

\begin{IEEEbiography}[{\includegraphics[width=1in,height=1.25in,clip,keepaspectratio]{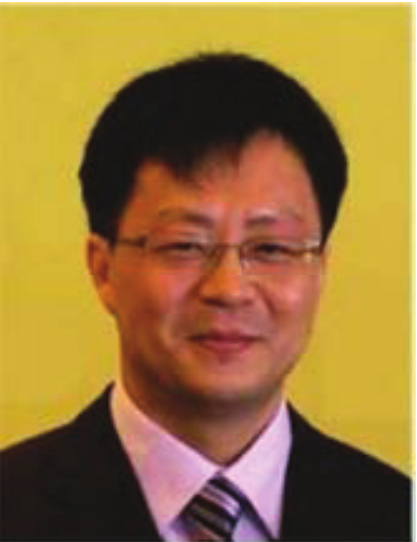}}]{Zhang Yi} (SM'10--F'15)
received the Ph.D. degree in mathematics from the Institute of Mathematics, The Chinese Academy of Science, Beijing, China, in 1994. Currently, he is a Professor at the Machine Intelligence Laboratory, College of Computer Science, Sichuan University, Chengdu, China. He is the co-author of three books: Convergence Analysis of Recurrent Neural Networks (Kluwer Academic Publishers, 2004), Neural Networks: Computational Models and Applications (Springer, 2007), and Subspace Learning of Neural Networks (CRC Press, 2010). He was an Associate Editor of IEEE Transactions on Neural Networks and Learning Systems (2009--2012) and is an Associate Editor of IEEE Transactions on Cybernetics (2014~). His current research interests include Neural Networks and Big Data. He is a fellow of IEEE.
\end{IEEEbiography}

\begin{IEEEbiography}[{\includegraphics[width=1in,height=1.25in,clip,keepaspectratio]{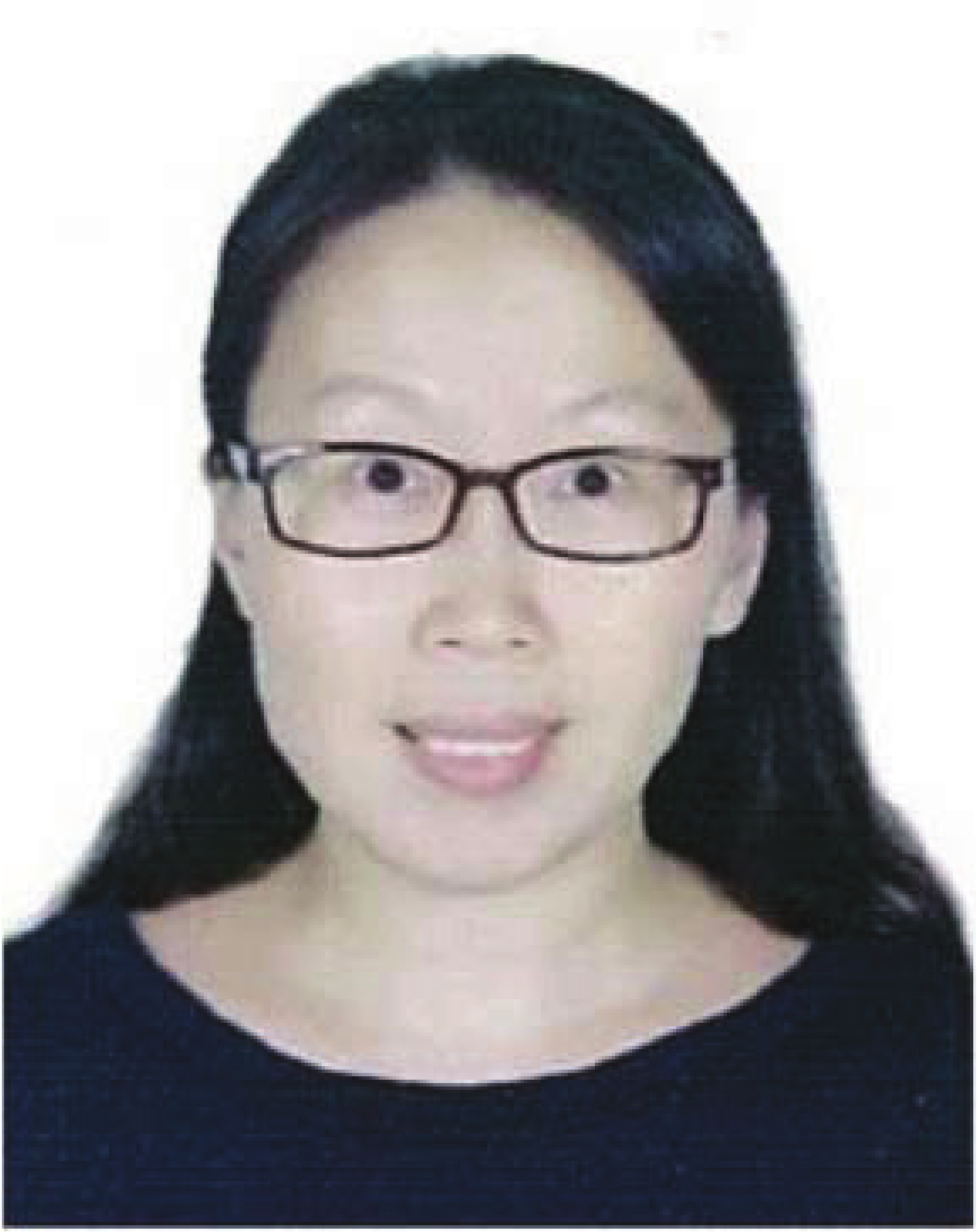}}]{Rui Yan} (M'11) received the bachelor’s and
master’s degrees from the Department of Mathematics, Sichuan University, Chengdu, China, in 1998 and 2001, respectively, and the Ph.D. degree from the Department of Electrical and Computer Engineering, National University of Singapore, Singapore, in 2006. She is a Professor with the College of Computer Science, Sichuan University, Chengdu, China. Her current research interests include intelligent robots, neural computation and nonlinear control.
\end{IEEEbiography}

\end{document}